%% file: bare_adv.tex
\documentclass[10pt,journal,compsoc]{IEEEtran}
\usepackage{colortbl}
\usepackage{pifont}
\usepackage[mathscr]{eucal}
\usepackage[mathletters]{ucs}
\usepackage[utf8x]{inputenc}
\usepackage[T1]{fontenc}
\usepackage{tikz}
\newcommand*{\circled}[1]{\lower.7ex\hbox{\tikz\draw (0pt, 0pt)%
    circle (.3em) node {\makebox[1em][c]{\small #1}};}}
\usepackage{algorithm}
\usepackage{algorithmic}
\usepackage{graphicx}
\usepackage{subfigure}
\usepackage{caption}
\usepackage{amsfonts}
\usepackage{amsmath}
\usepackage{amsthm}
\newtheorem{thm}{Theorem}[section]

\newtheorem{defn}{Definition}[section]
\theoremstyle{definition}

\ifCLASSOPTIONcompsoc
  \usepackage[nocompress]{cite}
\else
  \usepackage{cite}
\fi
\ifCLASSINFOpdf
\else
\fi
\usepackage{amsmath}
\begin{document}
%
\title{Roulette: A Semantic Privacy-Preserving
Device-Edge Collaborative Inference Framework for
Deep Learning Classification Tasks}
%
%
%
%

\author{Jingyi Li,
        Guocheng Liao,
        Lin Chen,
        and Xu Chen
\IEEEcompsocitemizethanks{\IEEEcompsocthanksitem Jingyi Li, Lin Chen and Xu Chen are with School of Computer Science and Engineering, Sun Yat-sen University, Guangzhou, China.\protect\\
\IEEEcompsocthanksitem Guocheng Liao is with School of Software Engineering, Sun Yat-sen University, Zhuhai, China. }}
\IEEEtitleabstractindextext{%
\begin{abstract}

Deep learning classifiers are crucial in the age of artificial intelligence. The device-edge-based collaborative inference has been widely adopted as an efficient framework for promoting its applications in IoT and 5G/6G networks. However, it suffers from accuracy degradation under non-i.i.d. data distribution and privacy disclosure. For accuracy degradation, direct use of transfer learning and split learning is high cost and privacy issues remain. For privacy disclosure, cryptography-based approaches lead to a huge overhead. Other lightweight methods assume that the ground truth is non-sensitive and can be exposed. But for many applications, the ground truth is the user's crucial privacy-sensitive information. In this paper, we propose a framework of Roulette, which is a task-oriented semantic privacy-preserving collaborative inference framework for deep learning classifiers. More than input data, we treat the ground truth of the data as private information. We develop a novel paradigm of split learning where the back-end DNN is frozen and the front-end DNN is retrained to be both a feature extractor and an encryptor. Moreover, we provide a  differential privacy guarantee and analyze the hardness of ground truth inference attacks. To validate the proposed Roulette, we conduct extensive performance evaluations using realistic datasets, which demonstrate that Roulette can effectively defend against various attacks and meanwhile achieve good model accuracy. In a situation where the non-i.i.d. is very severe, Roulette improves the inference accuracy by 21\% averaged over benchmarks, while making the accuracy of 
discrimination attacks almost equivalent to random guessing.

\end{abstract}

}

\maketitle

\IEEEdisplaynontitleabstractindextext

%
\IEEEpeerreviewmaketitle

\input{all}
\bibliographystyle{IEEEtran}
\bibliography{mybib.bib}
\end{document}

%% file: all.tex
\input{intro}

\input{preliminary}

\input{fram}
\input{SecAnalt}
\input{eval}

\input{rela}

\input{conclusion}

\input{appendix}

%% file: intro.tex
\section{Introduction}
\subsection{Background and Motivation}
\IEEEPARstart{O}{ver} the past decade, deep learning (DL) and deep neural networks (DNN) have been in full swing. As one of its most important branches, deep learning classifier has performed excellently in the fields of computer vision, natural language processing, robotics, etc \cite{AI,AI2}. In recent years, enabled by advanced communication technologies such as 5G/6G, artificial intelligence of things (AIoT) applications are emerging, and many devices (e.g., cameras, robots, and vehicles) have a strong demand for real-time DL classification services \cite{AIoT}. 


However, limited resources of the devices hinder the wide promotion of AIoT. Specifically, the DL model sizes are becoming larger and more computation-intensive, but the end devices are generally resource-constrained and tasks executed on them are typically delay-sensitive (telemedicine, autonomous driving, etc.), which require a massive amount of computing resources \cite{EdgeAI}. One promising approach to make real-time DNN inference on resource-constrained devices feasible is device-edge-based collaborative inference (co-inference) \cite{co-inference}, by leveraging the principle of edge computing. For co-inference, the pre-trained DNN is partitioned into two parts: the front-end part and the back-end part. And the device executes the front-end part locally and then offloads the corresponding intermediate representation to an edge server in proximity which executes the back-end part via high-quality wireless access enabled by 5G/6G networks. Co-inference is widely regarded as a key building block for distributed DNN task processing systems and has been intensively studied from the perspectives of execution speed acceleration and cost optimization \cite{CoEdge2,CoEdge3,CoEdge4}. However, there are still two critical issues that are much less understood in the literature. 

On one hand, when the local data of the devices and the original dataset that pre-trains the DNN are not independent and identically distributed (non-iid), and the number of the training samples is limited, model accuracy will drop significantly \cite{non-iid}. A straightforward method is to prepare a model for each distribution by the server in advance.  Further, two wiser methods to solve the non-iid issue are transfer learning \cite{transfer1, transfer2} and split learning \cite{split}. However, all these methods typically require a unique customized DNN model or multiple models\footnote{A single user's data may span multiple distributions, and different users may share the same distribution} for different individual users which would incur high storage and computing costs and low scalability. Thus, preparing a model for each distribution by the server in advance, transfer learning, and split learning cannot be directly applied in resource-constrained AIoT applications without adaptation.

On the other hand, co-inference is vulnerable to privacy attacks, which could lead to significant information leakage. Particularly, the intermediate representation of DNN offloaded from the device to the edge server actually contains a lot of feature information of data from the devices. The edge server or any malicious attackers after intercepting the intermediate representation, after receiving the intermediate representation, can launch model inversion attacks (i.e., instance recovery attack) to reconstruct raw data \cite{InversionAttack1,InversionAttack2,privacy-issue-add}. Also, the membership inference attack \cite{MembershipAttack1,MembershipAttack2} can be launched to infer whether an instance is in the training dataset. Indeed, many countries and regions have enacted privacy protection-related laws, e.g., General Data Protection Regulation (GDPR) in European Union. Therefore, device-edge-based co-inference services are highly desirable to resolve critical privacy issues.

Traditional cryptography-based privacy-preserving solutions, such as homomorphic encryption (HE) \cite{HE}, secure multi-party computing (SMC) \cite{SMC}, trusted execution environment (TEE) \cite{TEE}, etc., can guarantee full privacy protection, but these solutions either introduce severe computational delay or require special equipment to support \cite{HE1,HE3,SMC1,co-inference-add}. So it is overhead-prohibited and impractical to apply them in many application scenarios of AIoT. Also, these solutions are \textbf{\textit{task-independent}} and oblivious to what meaning the bits convey or how they would be used.

To prevent or limit privacy leakage, a handful of solutions exploit lightweight privacy-preserving DNN inference recently. Most of them assume that the primary information related to the task of DNN inference is regarded as non-sensitive content that can be exposed, but the extra information other than that is considered to be private. Hence, the goal of these works is to minimize the extra information while maintaining the primary information. For example, given a human face dataset and a gender classifier trained on the dataset, these works regard the gender information as non-sensitive content but extra information, e.g., hairstyle, as private information. To achieve this goal, different kinds of approaches are proposed, such as dynamic noise injection \cite{Shredder, Cloak}, adversarial training \cite{GANdefense1,GANdegense2, IntermRep2}. These solutions are \textbf{\textit{task-oblivious}}, that is, the task output of the classifier is not protected at all. However, these output predictions themselves are usually private information of the device's local data and downstream task processing,  which would risk significant privacy leakage. From the perspective of DL classification tasks, we coin this as a semantic privacy protection issue, since for many application scenarios the input context for a DL classifier can be images, videos, or texts and its output is a classified label to represent the key semantic meaning (e.g., image recognition and object detection) or the action/response under the relevant semantic setting (e.g., vision-based UAV navigation control). 

Motivated by the above two critical issues, in this paper, we aim at answering the following key question: \textbf{Can we design a \textit{task-oriented} privacy-preserving co-inference framework that provides semantic privacy for deep learning classifier and achieves acceptable model accuracy?} 

\subsection{Framework Sketch}
In this paper, we propose Roulette\footnote{We name the proposed framework Roulette because we break the original data-label correspondence randomly and inject noises during the operation, and the uncertainty philosophy is similar to the Roulette game where the numbers are randomly distributed around the circle and the artificial rotation is "noisy".}, a semantic privacy-preserving co-inference framework for deep learning classifier. Specifically, to deal with non-iid issues, we build on traditional split learning and retrain the front-end DNN parameters using a local dataset of the devices. In this way, we obtain a device-specific front-end DNN model to defend against accuracy degradation due to non-iid, instead of using a pre-trained front-end model. As for the privacy issue, in this paper, more than input data, we treat the ground truth of the data at the device in terms of classification task as sensitive semantic privacy that we should protect. To this end, different from traditional split learning that trains the entire model as a whole, we keep back-end model at the edge server fixed as the deterministic function. And we then \textbf{train the front-end DNN part to be able to map inputs with a ground truth $y$ to the intermediate representations which can be classified by the deterministic function, i.e., the back-end part at the edge server, as another ground truth besides $y$.} In other words, the front-end model is trained to be not only a feature extractor but also an encryptor. To do this, we replace the original data-label correspondence with randomly generated mapping before training and use the new mapping to compute classification loss (i.e., cross-entropy \cite{CEL}) during training. Meanwhile, to enforce the distribution of the intermediate representation with different mappings to be consistent and indistinguishable, we introduce a novel distance-minimizing term into the loss function to minimize the difference between the original and new intermediate output. Further, we also leverage the differential privacy mechanism which adds random noise to intermediate results. 

The main advantage of such split learning in terms of security is that it can protect the information about the device's specific model architecture and parameters against the edge server and other adversaries, which makes it the server hard to infer any information and thus, protects the privacy, but the original split learning does not provide any privacy theoretical guarantees. Instead, in this paper, we provide differential privacy guarantees and prove the hardness of ground truth inference attacks given the intermediate representation and local model parameters. It is worth noting that the training phase is offline, and the training of the front-end part can be performed on the device owner's private server rather than directly on the device in order to accelerate the training process. In the online inference phase, the device executes the front-end part to generate an encrypted intermediate representation (instead of the actual intermediate representation) and offloads it to the edge server. 


To illustrate our framework more vividly, we give an example of inference phase operation, as shown in Fig 1. Consider a multi-class classification task with classes "dog" and "cat". Suppose a data instance has the true label "cat". Pass the instance into the re-trained front-end model to generate an intermediate representation. The server then gets the inference result "dog" from the intermediate representation instead of "cat", and sends the "dog" back to the devices. Then the device looks up the pre-defined mapping table to obtain the real result "cat''. 

\begin{figure}[h]
  \centering
  \includegraphics[width=\linewidth]{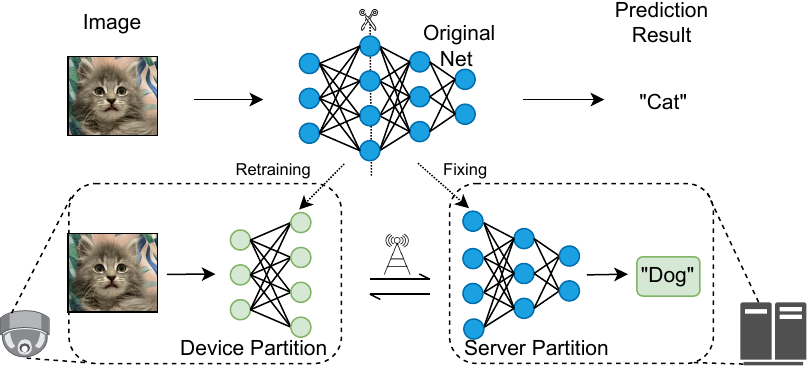}
  \caption{Illustration of the key idea in Roulette.}
  \label{fig:schematic}
  \vspace{-1em}
\end{figure}





\subsection{Key Contributions}
The key contributions of this paper are articulated as follows:

\begin{itemize}
\item{We propose a novel semantic privacy-preserving co-inference framework of  Roulette for device-edge based deep learning classification tasks which has significant advantages. First, the framework achieves higher accuracy compared with naive co-inference under non-i.i.d. Second, the framework allows device-specific front-end model retraining tailored to local data heterogeneity, and is also friendly for practical deployment such that the edge server can just host a common back-end model to serve multiple heterogeneous devices. }
\item{We propose a two-fold privacy protection method. First, we train the front-end model to be not only a feature extractor but also an encryptor to hide the true ground from the server and prove the hardness of ground truth inference attack. Second, we propose a privacy-preserving mechanism for interfering with local data transformations based on differential privacy, and the corresponding privacy guarantee is analyzed .}
\item{We conduct thorough performance evaluations using realistic datasets. We observe that the defense ability of the proposed Roulette framework against privacy attacks is outstanding, e.g., it can effectively thwart the model inversion and shadow model attack, and the differential privacy budget is relatively small. Meanwhile, Roulette can also achieve a good model accuracy under non-iid data distributions.}

\end{itemize}

The rest of the paper is organized as follows. First, we give some preliminaries in Sec. \ref{preliminary}. Then, we present the detailed introduction to the proposed framework in Sec. \ref{fram}. Security analysis is conducted to formally quantify the security of Roulette in Section F. We show evaluation in Sec. \ref{eval}. We review the related works in Sec. \ref{rela}.  Finally, we discuss and conclude in Sec. \ref{conclusion}.

%% file: preliminary.tex
\section{Technical Preliminaries\label{preliminary}}

\subsection{Split Learning and Co-inference}
\textbf{DNN:} A DNN is a parameterized function $f_\theta$ composed of several different types of layers, including convolution layer, pooling layer, fully connected layer, etc. There are multiple neurons in each layer, and the neurons take in the weighted tensor from the front layers and generate activation for the next layers. To apply DNN, a training set $\mathcal{D} = \{ x^{(i)}, y^{(i)} \}_{i=1}^{M}$ is needed, and the optimal parameter $\theta^*$ matching the reflection between $\mathcal{X}$ and $\mathcal{Y}$ supported by $\{x^{(i)}\}_{i=0}^M$ and $\{y^{(i)}\}_{i=0}^M$ is obtained by iterative training on $\mathcal{D}$ through back propagation and stochastic gradient descent, i.e.,
\begin{equation}
  \theta^*= \arg \min\limits_{\theta} (\sum_{i=1}^{M}\mathcal{L}(y^{(i)}, f_{\theta}(x^{(i)})),
\end{equation}
where $\mathcal{L}$ is loss function that measures the degree of deviation between the output of $f_{\theta}$ and the ground truth. The backpropagation procedure calculates the partial derivative of $\mathcal{L}$ with respect to each parameter $\theta_j$ in $\theta$, and then takes an average value over the batch, i.e., $g_j=\frac{1}{S}\Sigma_i\nabla_{\theta_j}\mathcal{L}(y^{(i)}, f_{\theta}(x^{(i)}))$, where $S$ is batch size. The update rule for $\theta_j$ is:
\begin{equation}
  \theta_j=\theta_j-\alpha g_j,
\end{equation}
where $\alpha$ is the learning rate. After training, given a $x$, the DNN can produce a inference result $y=f_{\theta^*}(x)$. For a more detailed introduction to DNN, please refer to \cite{AI, AI2, deeplearning}.

\textbf{Split learning:} Split learning realizes distributed learning by partitioning DNN model into two continuous chunks between the device and edge server. It has gained particular interest due to its efficiency and
simplicity \cite{unleashing}. In split learning, the devices own the
first partition $f_{\theta_1}$ of the model, whereas the edge server maintains the
remaining neural network $f_{\theta_2}$. The devices model’s
architecture and hyper-parameters are decided by the devices
before the training phase. In particular, they agree on a suitable
partition of the deep learning model and send the necessary information to the edge server. The edge server has no decisional power and
ignores the first partition deployed on the devices \cite{distributed_upgrade}.  In split learning, training is performed through a vertically distributed back-propagation. The upgrade of the devices model's parameters is based on the partial derivative sent by the edge server, which is sketched in Fig. \ref{fig:split}. In the case of supervised loss functions, split learning requires
the devices to share the labels with the edge server. To prevent label leakage, loss calculation can also be performed on the side of the device. The edge server sends the output (logits) of $f_{\theta_2}$ to the devices, and the devices calculate the loss and upload it to the edge server.

\begin{figure}[ht]
  \centering
  \includegraphics[width=\linewidth]{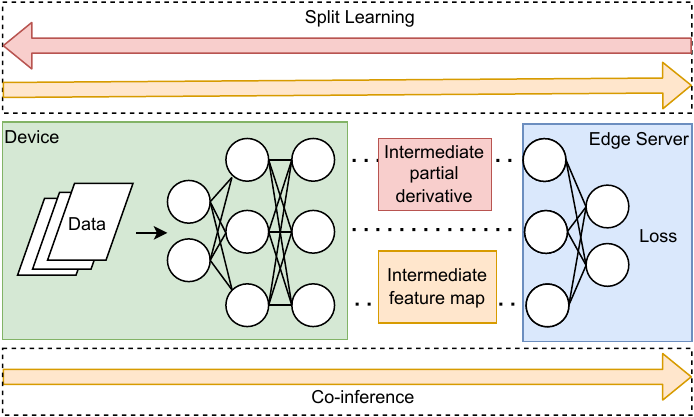}
  \caption{Diagram of split learning and co-inference. }
  \label{fig:split}
\end{figure}

\textbf{Co-inference:} As a prominent solution for fast edge inference, collaborative inference has been widely studied. The basic idea is to partition DNN into multiple partitions, with each part allocated to different participants. In co-inference, a pre-trained DNN $f_{\theta}$ are partitioned into $n$ partitions $ f_{\theta_1} \circ f_{\theta_2} \circ \cdots \circ f_{\theta_n}$ and delivered to $n$ different participants. Given a input $x$, the inference result is $y = f_{\theta_n}(...f_{\theta_2}(f_{\theta_1}(x))...) = f_{\theta}(x)$. Without losing generality, in this paper, we only study the case of $n = 2$. The relationship between split learning and co-inference is shown in Fig. \ref{fig:split}.

In the literature on privacy-preserving co-inference,  some stick to the assumption that the whole model is given, privacy enhancement cannot modify the structure and the fixed parameters of the pre-trained DNN for the primary task, which is called Inference-as-a-service (INFaas) \cite{INFaaS1, INFaaS12,2021unsupervised,Shredder,Cloak}. Others, on the contrary, redesign brand new forms of models $g_{\psi}$ satisfying $g_{\psi}(x) \approx f_{\theta}(x)$ for co-inference to preserve the privacy of $x$, e.g., MixCon \cite{MixCon}, DataMix \cite{DataMix}, Complex Valued Network \cite{ConplexedVluedNetwork}. In our framework, we adopt a half-half strategy that the edge server uses the fixed pre-trained model, while the device retrains a new model, instead of using the original pre-trained model. Our situation is in line with the mode of INFaas but retrains the partition on the device giving full play to scalability. 

\subsection{Differential Privacy}
Differential privacy is a mathematical framework defined for privacy-preserving data analysis, which aims at providing privacy guarantee for sensitive data and is regarded as a standard notion for rigorous privacy \cite{DP2}. The formal definition of $\epsilon$-differential privacy is as follows.

\begin{defn}
\textit{($\epsilon$-differential privacy)}.  Given two neighboring inputs $D$ and $D^{\prime}$ which differ in only one item, a randomized algorithm $\mathcal{M}$ satisfies $\epsilon$-differential privacy if $Pr[\mathcal{M}(D)\in S] \leq e^{\epsilon}\cdot Pr[\mathcal{M}(D^{\prime})\in S]$, where $S$ is any output set of $\mathcal{M(\cdot)}$.
\label{def:1}
\end{defn}

Here, the parameter $\epsilon$ is the privacy budget that trades off the utility and privacy of the output of the algorithm $\mathcal{M}$. A smaller $\epsilon$ indicates stronger privacy protection.  According to this definition, a differentially private algorithm can provide aggregate representations about a set of data items without leaking information of any data
item.

To achieve differential privacy, the common technique is to inject Laplacian noise to the output of the deterministic function $f$:
\begin{equation}
\label{lapM}
\hat{f}(D) = f(D) + Lap \left(\frac{\Delta f}{\epsilon }\right),
\end{equation}
where $\Delta f$ is the global sensitivity of $f$, which is defined as the maximum $L_1$ distance between the results of the function of two neighboring inputs. And, $Lap(\frac{\Delta f}{\epsilon})$ is a random variable sampled from the Laplace distribution with scale $\frac{\Delta f}{\epsilon}$.




Differential privacy has a superb property of being immune
to post-processing \cite{DP1}, which means that any post-processing algorithm will not incur additional privacy leakage.

\subsection{Non-i.i.d. issues of co-inference}
One of the significant challenges of co-inference is the non-i.i.d. between the local data and the original training data, which deteriorates the model accuracy. There are two kinds of non-i.i.d issues in literature, \textit{label skew} and  \textit{attribute distribution skew } \cite{non-iid}.  Label skew means that label distribution is different from the original training dataset. The second kind is \textit{attribute distribution skew}. Attribute distribution skew refers to the attributes of the data sample with the same label being different. For example, the standard MNIST \cite{MNIST} and SVHN \cite{svhn_paper} datasets are ”numbers” labeled 0-9, but the data samples in MNIST are clear and simple and samples in SVHN are accompanied by more interference, noise, and rotation. In this paper, we will tackle this challenge by local retraining part of the network.

\subsection{Threat model}

In our model, there are two parties, devices and an edge server. In order to enable split learning and co-inference, the device has to upload the intermediate representation generated through the DNN deployed on it and loss value for the learning task. In our setting, the server is an \textit{honest but curious} adversary. We assume the server would not manipulate and hijack the learning process of the device’s DNN model \cite{unleashing}. However, an adversary would intercept the intermediate representation
and perform any attacks, e.g., instance recovery attack \cite{InversionAttack1, InversionAttack2}, membership inference attack \cite{MembershipAttack1, MembershipAttack2} and property inference attack \cite{Property}.

%% file: fram.tex
\section{The Roulette framework \label{fram}}

We first give a high-level overview of Roulette and then present its components and the technical design, in Section \ref{A}. Roulette is an end-to-end non-invasive solution consisting of two disjoint phases: offline split learning and online co-inference, shown in Fig. \ref{fig:fig3}. The training phase consists of three main parts: forward propagation, loss computation, and backpropagation, which will be elaborated in Section  \ref{B},  Section  \ref{C},  and Section  \ref{D}, respectively. We introduce the online co-inference phase in Section  \ref{E}. 

\subsection{Roulette Overview\label{A}}
\begin{figure*}[t]
    \centering
    \subfigure[Split learning]{
        \begin{minipage}[t]{0.48\linewidth}
        \centering
        \includegraphics[width=\linewidth]{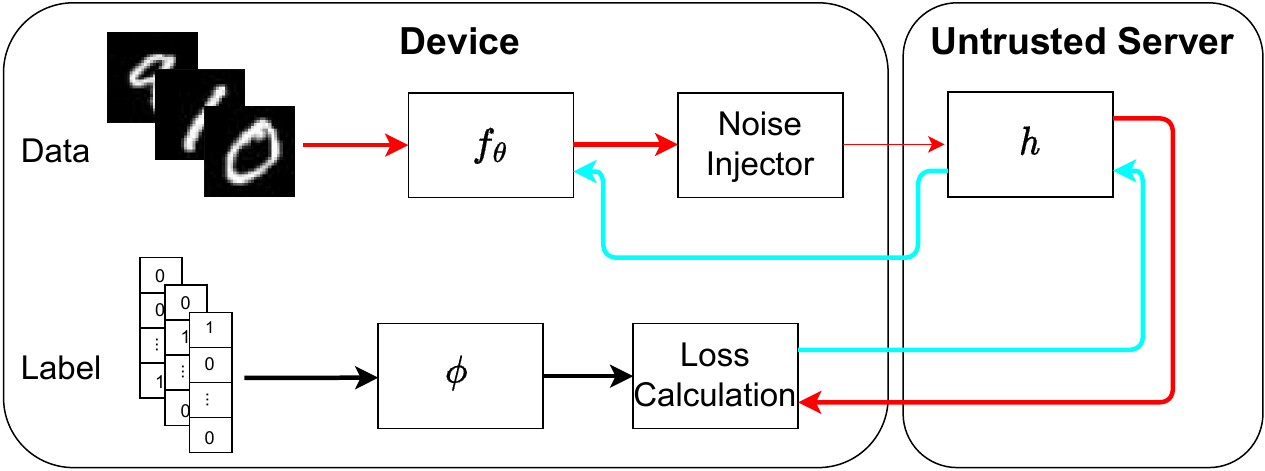}
        \label{fig:fig3a}
        \end{minipage}
    }
    \subfigure[Co-inference]{
        \begin{minipage}[t]{0.48\linewidth}
        \centering
        \includegraphics[width=\linewidth]{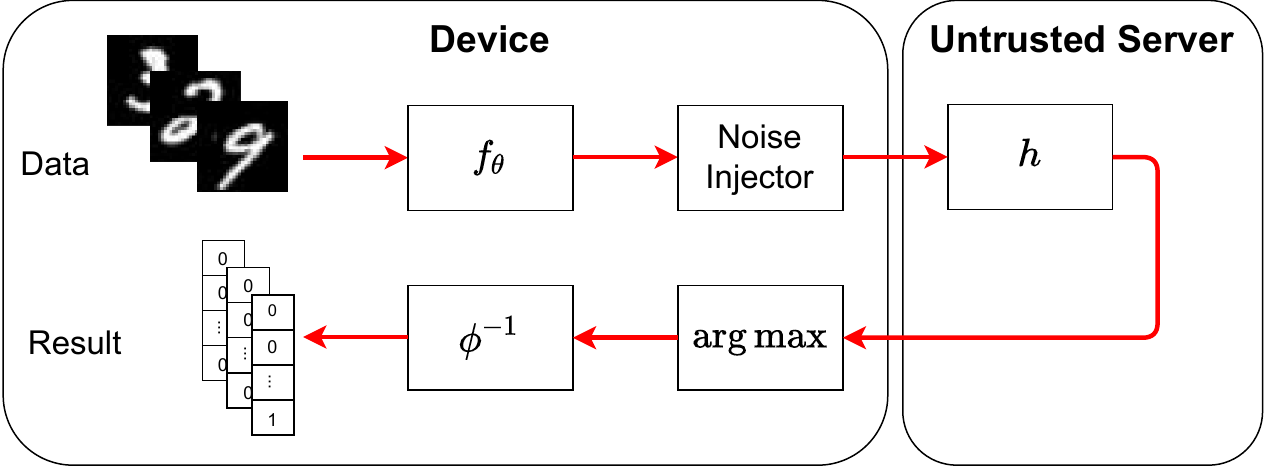}
        \label{fig:fig3b}
        \end{minipage}
    }
    \caption{Workflow of a batch of data: The red arrow and the bright blue arrow represent forward propagation and backpropagation data streams, respectively.}
    \label{fig:fig3}
    \vspace{-1em}
\end{figure*}

First of all, we introduce some necessary notations of the framework. Let the label (ground truth) range of an $N$-class classification task be $\mathcal{N}_g=\{0,1,...,N-1\}$. Consider a pre-trained $N$-class classification DNN model $M_{\psi^*}:\mathbb{R}^{p}\rightarrow \mathcal{N}_g$, where $p$ is the flattened input size, and $\psi^*$ is the parameters of the model that fully determines the function. It is worth noting that the actual output of the $N$-class classification DNN $y=M_{\psi^*}(x)$ is a $N$ dimensional one-hot vector, i.e., $y\in \mathbb{R}^{N}$, but for convenience, we use a scalar indicator to represent it. The edge server negotiates with the devices to partition $M_{\psi^*}$ into $M_\psi^* = f_{\psi^*} \circ h_{\psi^*}$, $f_{\psi^*}:\mathbb{R}^{p}\rightarrow \mathbb{R}^{o}$, $h_{\psi^*}:\mathbb{R}^{o}\rightarrow \mathcal{N}_g$, and deploy $h_{\psi^*}$ on the edge server, where $o$ is the flatten dimension of the intermediate representation. It is worth noting that the parameters of the device partition are trained and updated, and the edge server partition is always unchanged. Since $M_{\psi^*}$ will not change after the partition is completed, for simplicity, $h_{\psi^*}$ is denoted as $h^*$, and $f_{\psi^*}$ is denoted as $f^{*}$ as well in the remainder of this paper. The device redesigns a model $f_{\theta}:\mathbb{R}^{p}\rightarrow \mathbb{R}^{o}$ whose input size is $p$ and output size is identical to the input size of $h^*(\cdot)$, where $\theta$ represents the trainable parameters. We denote $\theta_t$ as the parameters of the $t$th-round of training. It is worth noting that the users can just choose the appropriate partition point and make sure the dimension of the uploaded vector is valid (identical to the input dimension of $h^*(\cdot)$). The structure of $f_\theta$ can be decided completely on their own according to their devices' capacity and the users can also use lightweight AutoML \cite{automl} tools to generate the local structure automatically. We denote bijection function $\phi:\mathcal{N}_g\rightarrow\mathcal{N}_g$ as encryption key and its inverse function $\phi^{-1}$ as decryption key. 

Recall the challenge that the local data is non-i.i.d. with the dataset for pre-training the original end-part model, which is one of the heterogeneous natures of the users. Further, the capacity of each user's device is various, i.e., different memory, CPU/GPU frequency, and battery status. Roulette allows each user to choose the partition point and design their own \textit{device-specific} \textit{personalized} local model (the front part) structure. Specifically, each user can self-define $f_\theta$ by jointly considering the capacity of the local device and volume of the datasets. Then, the user performs multi-round fine-tuning to update the parameters of the front-end model. Finally, the user will obtain a personalized device-specific local model.

Then, we illustrate the high-level overview of the workflow of a batch of data in Fig. \ref{fig:fig3}. In the training phase, the devices randomly sample a batch of data $x$ from the local training dataset and generate differential private intermediate representation through $f_{\theta_t}(x)$ and noise injector. Next, the noisy intermediate representation is uploaded to the untrusted edge server. Sec. \ref{B} will introduce the detailed differentially private transformation. The server then feeds the intermediate representation to the back-end of pre-trained DNN $h^{*}$, obtains the logits output, and sends the logits to the devices. The above forward-propagation process is marked by a red arrow in Fig. \ref{fig:fig3a}. Each device then encrypts the labels corresponding to the sampled data according to their own encryption key $\phi$ and compute loss function. The detailed encryption generation and loss function design is explained in Sec. \ref{C}. Then, the devices send the loss value to the server, and the server back-propagates the error to $f_{\theta_t}$. Lastly, the devices update the parameters $\theta_t$ to $\theta_{t+1}$. The back-propagation data flow is marked by a bright blue arrow in \ref{fig:fig3a}. The online inference phase is shown in \ref{fig:fig3b}. The procedure of inference is similar to forward-propagation in the training phase except that at the end of inference, each device decrypts the result by the decryption key $\phi^{-1}$.

\subsection{Forward Propagation\label{B}}
The red arrow in Fig. \ref{fig:fig3a} represents the forward propagation data flow and forward propagation is completed by the interaction between the device and the edge server, which will be introduced in the next.

\textbf{Data Transformer with Device Partition.} To extract the training (inputs, output) pairs for learning the device partition's parameters, a batch of the training data $x_l$ is supposed to be fed to $f_{\theta_t}$. Inspired by the previous conclusion that dropout is a relatively efficient approach to defense against model inversion attack\cite{InversionAttack2} and the earlier dropout, the better defense effects\cite{DP_DL}, we mask the input data with uniformly randomly generated nullification matrix $I_p\stackrel{R}\gets\{0,1\}^p$. And then, the nullified data $x_l^{\prime} = f_{\theta_t}(I_p\odot x_l)$ is attained, where $\odot$ refers to element-wise multiplication of two matrices and the number of zeros in $I_p$ as a percentage of the total number $p$ is the expected nullification rate $\eta$.

\textbf{Noise Injector.} Some works (e.g., \cite{Shredder, Property}) indicated that specific properties or whether an instance in the dataset may be inferred from the intermediate representation given some background knowledge. To alleviate the inference attacks, by employing the strong notion of differential privacy, we design a DP-guaranteed noise injector before sending the intermediate representation to the server. The noise injector adds the noise tensor to the intermediate representation, element-wise. One intuitive attempt of providing $\epsilon$-differential privacy is to add noise that follows the Laplace distribution with scale $\frac{\Delta f_{\theta_t}}{\epsilon}  $ into the output $x_r$ as (\ref{lapM}), where $\epsilon$ is the privacy budget parameter. However, it is difficult to estimate the global sensitivity of a DNN, and we use a clipping process (i.e., $x_r \leftarrow \max(1,\Vert x_r \Vert/B))$ to bound the value of $\Vert x_r \Vert$. It indicates that $x_r$ is preserved when ${\lVert x_r\rVert}_\infty\leq B$, whereas it is scaled down to $B$ when ${\lVert x_r\rVert}_\infty > B$. Then, the global sensitivity can be estimated as $2B$. The bound threshold $B$ is input-independent and so would leak no sensitive information. The bound $B$ can be set as the median of the infinity norm with regard to a batch of training examples \cite{DP_DL}. Further, \cite{IntermRep2} indicates that the Laplace mechanism can be applied at any local layers in the mode of split learning. Even though adding noise into the layer before the last local layer may result in worse accuracy, the differential privacy budget can be tighter by this operation. Our framework allows retraining, which can alleviate accuracy elimination. Therefore, we also take this generalization into account. We denote $f_{\theta_t}=f_{\theta_t}^{(1)}\circ f_{\theta_t}^{(2)}$. $f_{\theta_t}^{(1)}$ and $f_{\theta_t}^{(2)}$ refer to the layers before and after noising adding, respectively. Algorithm \ref{alg:DP} outlines the whole differentially private data transformation on the device side. The formal quantitative differential privacy budget will be given in Sec. F.

In our framework, the dimension of the mechanism's output is relatively much smaller (at the third conv. cut point of the LeNet \cite{lenet}, it is 400) than the method injecting noises into the gradient (around 60000 for the LeNet), thus the strength of the Gaussian mechanism \cite{DP1} which is widely applied in the differentially private DNN training \cite{DP2} cannot be embodied. Further, specific network structure, batch size, and partition point decide the dimension of the mechanism's output. Without losing generality, we employ the Laplace mechanism in this paper for simplicity, but the Gaussian mechanism can also be extended easily.

\begin{algorithm}[ht]
\renewcommand{\algorithmicrequire}{\textbf{Input:}}
\renewcommand{\algorithmicensure}{\textbf{Output:}}
\caption{ Differentially Private Transformation}
\label{alg:DP}
\begin{algorithmic}[1]
\REQUIRE Input data $x_l$;  Bound threshold $B$; Privacy
budget $\epsilon$; Nullification matrix $I_p$
\ENSURE Noisy intermediate representation $\widetilde{x}_r$.
\STATE $x_r\gets f_{\theta_t}^{(1)}(I_p\odot x_l)$
\STATE $d={\lVert x_r\rVert}_\infty$
\STATE $x_r\gets x_r/\max (1, \frac{d}{B})$
\STATE $x_r \gets x_r+Lap(2B/\epsilon)$
\STATE $\widetilde{x}_r\gets  f_{\theta_t}^{(2)}(x_r)$ 
\RETURN $\widetilde{x}_r$ 
\end{algorithmic}
\end{algorithm}


\textbf{Logits Generator with Server Partition.}
To be able to go through the learning process, a batch of logits which is the outputs of the layer before the last classification layer of the neural network, usually softmax, is required to pair with the labels. To generate the logits, the noisy intermediate representation is fed to the cloud partition, i.e,
\begin{equation}
   \hat{y}=h^{*}(\widetilde{x}_r).  
\end{equation}
And then, the edge server sends the logits to the device for loss calculation.

\subsection{Loss Function Calculation\label{C}}

In this subsection, we will elaborate on how encryption key $\phi$ generates and loss function design. 

\subsubsection{Encryption and Decryption Key Generation}
The encryption and decryption key generation is finished before the beginning of the training and will not change during training for each device. Recall that our front-end model $f_{\theta}$ should not only be a feature extractor but also an encryptor. The feature extractor function can be achieved by step-by-step parameters update based on normal classification error back-propagation. To achieve the encryptor function, we would like to break the original data-label correspondence and apply a randomly generated new one. In other words, we need to randomly generate a \textit{derangement}. To be specific, given the ground truth set $\mathcal{N}_g$, we denote \textit{group} formed by all \textit{permutations} (i.e., bijective functions from a set to itself) without repetitions and without fixed points ($\phi(x)=x$) of $\mathcal{N}_g$ as $\mathcal{S}$. Then, our aim is that randomly select a permutation from $\mathcal{S}$, i.e.,  $\phi\stackrel{R}\gets\mathcal{S}$, as the new data-label correspondence. $\phi$ plays a role of encryption key and the inverse function $\phi^{-1}$ is decryption key. To illustrate the mapping clearer, we give an example in Cauchy's two-line notation\cite{group}. A particular $\phi$ of the $\mathcal{N}_g=\{0,1,...,4\}$ can be written as
\begin{equation}
	\phi=\begin{pmatrix} 0 & 1 & 2 & 3 & 4 \\ 2 & 3 & 4 & 0 & 1  \end{pmatrix},
\end{equation}
this means that $\phi$ satisfies $\phi(0)=2,\ \phi(1)=3,\ \phi(2)=4,\ \phi(3)=0$, and $\phi^{-1}(2)=0,\ \phi^{-1}(3)=1,$ etc.

Further, it is impractical to form $\mathcal{S}$ first and then uniformly sample from $S$ in practice, because $|\mathcal{S}|=\Theta((N-1) !)$ (lower bounded by the number of \textit{circular permutations}). A naive approach is to use a standard shuffle algorithm first and then swap the elements that  do not change their positions with nearby elements. However, it is non-trivial to make the outcome confirm to the uniform distribution. Thanks to C. Mart{\'\i}nez et al. (2008) proposed an elegant algorithm for generating random derangements \cite{martinez2008generating}. We employ the elegant algorithm as our encryption key generator. The \texttt{KeyGen} algorithm is presented as Algorithm \ref{alg:KeyGen}, where
\begin{equation}
    D_i= \lfloor \frac{i!+1}{e} \rfloor,
\end{equation}
which is the exact number of derangements given $i$. The average cost of the algorithm is $e\cdot n+o(n)$ from the perspective of numbers of generating random numbers. Then, we will use the generated mappings (keys) to encrypt the local labels. The encryption ensures the complexity of the searching space, e.g., $|\mathcal{S}|=\Theta((N-1) !)$, of the shadow model attacker with background knowledge of dataset distribution but without exact model parameters. 

\begin{algorithm}[ht]
\renewcommand{\algorithmicrequire}{\textbf{Input:}}
\renewcommand{\algorithmicensure}{\textbf{Output:}}
\caption{ Encryption Key Generator}
\label{alg:KeyGen}
\begin{algorithmic}[1]
\STATE \textbf{function} \texttt{KeyGen}({$N$})
\STATE // $N$: total number of classes
\STATE $Key \gets [0,1,...,N-1]$
\FOR{$i=0$ to $N-1$} 
\STATE $mark[i]\gets false$
\ENDFOR
\STATE \ $i\gets n$;\ $u\gets n$
\WHILE{$u\ge 1$}
\IF{$\neg mark[i]$}
\STATE\textbf{repeat} $j\gets$RANDOM$(0,i-1)$\\
\STATE\textbf{until} $\neg mark[j]$\\
\STATE \texttt{Swap}($Key[i], \, Key[j]$)
\STATE $p\gets $UNIFORM$(0,1)$
\IF {$p<(u-1)D_{u-1}/D_{u}$}
\STATE $mark[j]\gets true$; $u \gets u-1$ 
\ENDIF
\STATE $u\gets u-1$
\ENDIF
\STATE $i\gets i-1$
\ENDWHILE
\RETURN $Key$
\STATE // return $\phi(\cdot)$
\end{algorithmic}
\end{algorithm}

\subsubsection{Loss Function and Training Method}
Our loss function involves two terms. The first term $\mathcal{L}_{class}$ is to calculate the classification error. Instead of taking the $\hat{y}$ and the corresponding labels $y$ directly, we encrypt $y$ by encryption key $\phi$ first, i.e., 

\begin{equation}
\label{eq:class}
    \mathcal{L}_{class}=-\frac{1}{M}\sum_{i=1}^{M}\sum_{j=1}^{N}\phi(y_j^{(i)})\log(\hat{y}_j^{(i)}),
\end{equation}
where $M$ is batch size and $y$ is the true label corresponding to the data $x_l$.

Further, to make the intermediate representation under different $\phi$ more indistinguishable, we need to add a term to minimize the distance between $\widetilde{x}_r$ and $f^*(x_l)$, i.e., 
\begin{equation}
    \mathcal{L}_{dist}=dist(\widetilde{x}_r, f^*(x_l)),
\end{equation}
where $dist(\cdot)$ is supposed to be the distance measurement function. However, it is difficult to find an appropriate explicit distance function due to intermediate representation space is generally high-dimensional, and the decision boundary is non-linear. We notice that the distance we concern actually depends on the classification decision boundary. Hence, it is very suitable to apply Generative Adversarial Network (GAN) \cite{GAN} to implicitly minimize the distance. To this end, we introduce an additional DNN  $D_{\pi}$ parameterized by $\pi$ to simulate a discriminator that distinguishes whether the intermediate representation is generated by the original network. Then, the $\mathcal{L}_{dist} $ can be written as,
\begin{equation}
\label{eq:distGAN}
    \mathcal{L}_{dist}=\frac{1}{M}\sum_{i=1}^{M}[-t\log(D_{\pi}(z^{(i)}))-(1-t)\log(1-D_{\pi}(z^{(i)})))],
\end{equation}
where $t$ is the intermediate representation type indicator, i.e.,
\begin{equation}
t=\left\{
\begin{aligned}
&1,    &&z^{(i)}= f^*(x_l)^{(i)};\\
&0,      &&z^{(i)}= \widetilde{x}_r^{(i)}.
\end{aligned}
\right.
\end{equation}
Therefore, the final loss function and the $\min$-$\max$ optimization problem can be written as,
\begin{equation}
    \min_{\theta}\ \max_{\pi}\quad \mathcal{L}_{cd}= \mathcal{L}_{class}+\mathcal{L}_{dist}.
\end{equation}
Algorithm \ref{alg:HTM} sketches the hybrid training method of Roulette. Before performing the Algorithm \ref{alg:HTM}, the device first pretrain the discriminator normally without task objective to obtain the best performance on distinguishing the intermediate representation type. Then within each training batch, the device first updates $\pi$ by maximizing $\mathcal{L}_{dist}$. Then, $\theta$ will be updated by minimizing $\mathcal{L}_{cd}$.
\begin{algorithm}[ht]
\renewcommand{\algorithmicrequire}{\textbf{Input:}}
\renewcommand{\algorithmicensure}{\textbf{Output:}}
\caption{ Hybrid Training Method}
\label{alg:HTM}
\begin{algorithmic}[1]
\REQUIRE Local data $\mathcal{D}$
\ENSURE $\theta$
\FOR{every epoch} 
\FOR{every batch}
\STATE $\mathcal{L}_{dist}\rightarrow$ update $\pi\ $ (maximizing $\mathcal{L}_{dist}$)
\STATE $\mathcal{L}_{cd}\rightarrow$ update $\theta\ $ (minimizing $\mathcal{L}_{cd}$)
\ENDFOR
\ENDFOR
\end{algorithmic}
\end{algorithm}

\subsection{Back Propagation\label{D}}

The update of $\pi$ can be done locally, but back-propagate the error to $f_\theta$ involving $h^*$. We notice that very little literature on split learning clarifies how back-propagation exactly works in split learning mode, so we elaborate on it first. Assuming the DNN $f_{\theta} \circ h^{*}$ has a total of $k$ layers and the partition point is layer $i$, i.e., $f_{\theta}=l_1 \circ l_2 \circ ... \circ l_i$ and $h^{*}= l_{i+1}\circ...\circ l_k$. We denote the trainable parameters of any $j$-th layer as $(W_{j}, \mathbf{b}_j)$, $j\in(1,k)$.The forward propagation operation for each convolutional or fully connected layer in feed-forward neural networks can be formalized as:
\begin{equation}
    \mathbf{x}_{j+1}=\sigma(W_{j}^{\mathsf{T}}\mathbf{x}_{j}+\mathbf{b}_{j}),
\end{equation}
where $\mathbf{x}_{j}$ refers to the input feature map of the $j$-th layer, $W_{j}^{\mathsf{T}}\mathbf{x}_{j}+\mathbf{b}_{j}$ is the affine transformation (including convolutional and fully connected layer) of the input vector. $\sigma$ is a non-linear layer including activation function, pooling, etc., or identity function (split before the non-linear layer). Then the partial derivative (gradient) of $j$-th layer's parameters respect to loss function $\mathcal{L}_{cd}$ can be calculated as:
\begin{equation}
\label{eq:ParaUp}
\begin{split}
    \nabla_{\mathbf{b}_{j}}\mathcal{L}_{cd}&=\nabla_{\mathbf{x}_{j+1}}\mathcal{L}_{cd}\odot\sigma^{\prime}(W_{j}^{\mathsf{T}}\mathbf{x}_{j}+\mathbf{b}_{j}),\\
    \nabla_{W_{j}}\mathcal{L}_{cd}&=\mathbf{x}_{j}(\nabla_{\mathbf{x}_{j+1}}\mathcal{L}_{cd}\odot\sigma^{\prime}(W_{j}^{\mathsf{T}}\mathbf{x}_{j}+\mathbf{b}_{j}))^{\mathsf{T}},\\
    \nabla_{\mathbf{x}_{j}}\mathcal{L}_{cd}&=W_{j}(\nabla_{\mathbf{x}_{j+1}}\mathcal{L}_{cd}\odot\sigma^{\prime}(W_{j}^{\mathsf{T}}\mathbf{x}_{j}+\mathbf{b}_{j})).
\end{split}
\end{equation}
It can be seen that the device owns $\mathbf{x}_{i}$ and $\sigma^{\prime}(W_{i}^{\mathsf{T}}\mathbf{x}_{i}+\mathbf{b}_{i})$ locally. The only necessary thing to update the $i$-th and before $i$-th layer's parameters from the server is $\nabla_{\mathbf{x}_{i+1}}\mathcal{L}_{cd}$, which is independent on the $f_{\theta}$. This is also the basis that the devices can devise their own structure freely. Sending $\mathcal{L}_{cd}$ to the server, the server can back-propagate the error to the spilled layer by $h^{*}$. Therefore, the back-propagation can be summarized as follows:

\textbf{Back Propagation with Server Partition.} After the loss calculation is completed, the loss value $\mathcal{L}_{cd}$ is sent to the edge server for back propagation. Then, the edge server back propagates gradient and the devices download the intermediate partial derivative $\nabla_{\mathbf{x}_{i+1}}\mathcal{L}_{cd}$.

\textbf{Back Propagation and Parameter Update with Device Partition.} 
After receiving $\nabla_{\mathbf{x}_{i+1}}\mathcal{L}_{cd}$, based on (2) and (\ref{eq:ParaUp}) the devices update the parameters of $f_{\theta}$.

\subsection{Online Co-Inference\label{E}}
After split learning, Roulette is prepared to do co-inference, which is depicted in Fig. \ref{fig:fig3b}. The device samples a batch of validated data and feds it into $f_{\theta}$, and then executes Algorithm \ref{alg:DP} to generate differentially private intermediate representation $\widetilde{x}_r$ and offloads it to the edge server. The edge server computes the back-end DNN and sent the logits $h^{*}(\widetilde{x}_r)$ back to the device. Finally, the device decrypts it and obtains the real prediction result according to
\begin{equation}
    \widetilde{y} = \phi^{-1}(\arg \max h^{*}(\widetilde{x}_r)). 
\end{equation}

We should emphasize that the front-end model training by the device can be done offline. The user can embrace devices with relatively more affluent resources than online inference. Also, offline training means the training procedure is not delay-sensitive. For instance, during the training stage, the user can perform the local training procedure on a private server  (interact with the remote server). Instead, for the deployment (inference stage), the trained local model has to be deployed on resource-constraint real-world devices (e.g., cameras, UAVs) for real-world applications. In this case, the real-time requirement of offline training is lower than the online inference stage.



%% file: SecAnalt.tex
\section{Security Analysis\label{F}}

Our privacy protection functions in a two-fold way: differential privacy and harness of ground truth inference attack.

\subsection{Differential Privacy}
We add random noise sampled from the Laplace distribution into the bounded output $x_r$ and dropout a proportion of elements of the input image to disturb data as differential privacy mechanism. One important issue of differentially private transformation is determining the privacy budget that represents a rigorous privacy guarantee provided by the designed mechanism. We formally prove the following result.

\begin{thm}
\label{thm:DPbudget}
Given the sensitive data $x_l$, the local DNN $f_{\theta_t}$ and nullification matrix $I_p$ with nullification rate $\eta$, Algorithm \ref{alg:DP} is $\varepsilon$-differentially private, where
\begin{equation}
    \varepsilon=\ln\left[(1-\eta)e^{\epsilon/\Xi} +\eta \right],
\end{equation}
and $\Xi=\lVert \nabla_{x_r}f_{\theta_t}^{(2)}\rVert_\infty$, where $x_r=f_{\theta_t}^{(1)}(x_l\odot I_p)$.
\end{thm}

See Appendix A for the proof of Theorem \ref{thm:DPbudget}. The privacy guarantee we state is on both re-training (split learning) and co-inference and it is element-wise event-level differential privacy\cite{DP1}. Differential privacy mechanism is a powerful guarantee that can eliminate a single element's shift on the output to hide information. Given an appropriate privacy budget, it can protect the data from membership inference attack\cite{MembershipAttack1,MembershipAttack2}.

\subsection{Hardness of ground truth inference attack}
Even though differential privacy provides a strict guarantee, the statement of avoiding holistic statistical privacy violation\cite{IsDP} is unclear. Our ground truth privacy exactly falls into this kind of violation. Hence, we need to analyze the ground truth privacy more carefully besides differential privacy. Next, we will analyze the ground truth privacy when the server has different level of background knowledge.

\textbf{Without Parameters and Without Dataset Distribution}. First of all, the privacy of inference results can be derived from the confidentiality of local parameters and the dataset of split learning. Several previous works with split learning also make such a statement \cite{split}. To our best knowledge, there is no case of a successful attack when the edge server knows neither device's local parameters nor dataset distribution. 

\textbf{With Parameters but Without Dataset Distribution}. Further, we investigate the stronger attacker that knows the local model's parameters but does not know the local dataset distribution. Such an attacker can perform a ground truth inference attack. We give the following definition:
\begin{defn}
Given the front-end model  $f_\theta$ and intermediate representation $x_r$, we say the attacker has a perfect ground truth inference function $c:\mathbb{R}^{p}\rightarrow \mathcal{N}_g$ that
\begin{equation}
    c(x^{\prime})=y,
\end{equation}
if and only if $x^{\prime}\in\{  x^{\prime}| f_\theta (x^\prime)=x_r\}$, where $y$ is the true label corresponding to input data.
\label{def:PerInf}
\end{defn}
Roughly speaking, Definition \ref{def:PerInf} implies that given $f_\theta$, if the server can find the input pattern that can be exactly matched to the received $x_r$, then the server is able to figure out the ground truth. This definition is too strict for the attacker in reality. One can imagine that even if the attacker finds an approximate input pattern that can be nearly matched to $x_r$ based on $f_\theta$, it can also link $x_r$ to the ground truth. Therefore, we give a relaxed version of Definition \ref{def:PerInf} as follows:

\begin{defn}
\label{def:RelaPerInf}
Given the front-end model  $f_\theta$ and intermediate representation $x_r$, we say the attacker has a $\gamma$-approximate ground truth inference function $c:\mathbb{R}^{p}\rightarrow \mathcal{N}_g$ that
\begin{equation}
    c(x^{\prime})=y,
\end{equation}
if and only if $x^{\prime}\in\{ x^{\prime}| \ \Vert f_\theta (x^\prime)-x_r\Vert_2\le\gamma\sqrt{o}\}$, where $y$ is the true label corresponding to input data.
\end{defn}

In Definition \ref{def:RelaPerInf}, the $\sqrt{o}$ factor is only for normalization and not essential. Based on the definition, we prove the hardness of the attacker to find a $\gamma$-approximate ground truth inference function. The theorem is formally stated below.

\begin{thm}
\label{thm:hardness}
There exists a constant $Q>1$, unless \textsf{RP=NP}, it is hard  to find a $\frac{1}{60Q}$-approximate ground truth inference function for the attacker.
\end{thm}

See Appendix B for the proof. Our proof is inspired by \cite{MixCon, InvHard}. In the proof, we reduce from \textsf{MAX3SAT} problem. Theorem \ref{thm:hardness} implies that by taking a suitable constant $Q>1$, it is hard for the attacker to approximate within some constant factor under Definition \ref{thm:hardness}.

\textbf{Without Parameters but With Dataset Distribution}. If the server knows the distribution of the local dataset but does not know $f_\theta$, it can traverse all possible $\phi$ to train the corresponding front-end model or extract some features. However, the searching space will be extremely huge when $N$ goes slightly big, because $|\mathcal{S}|=\Theta((N-1) !)$. Hence, it is intractable for the server to infer ground truth when the number of classification system outputs is big. When $N$ is small, we investigate the influence of the distribution difference between the server and the device experimentally in Sec. \ref{eval}.

\textbf{With Parameters and With Dataset Distribution.} In this case, the server can easily figure out the ground truth by matching the input sampled from the dataset distribution to the model output. But in reality, it is very difficult for the server to satisfy the background knowledge. In particular, the device has no motive to disclose its model parameters to the server.

Lastly, one may question whether we upload the loss value directly without privacy-preserving measurements. We clarify what ground truth privacy is regarding the classification task. Different from tasks like target tracking and object segmentation whose labels contain a lot of true content of the original data, labels of the classification tasks are serial numbers or one-hot vectors, which are meaningless without specific corresponding data. Therefore, our method aims at protecting original data-label correspondence instead of the label itself. This is also why we emphasize it is ground truth instead of the label. For example, for a classification system with ``dog'' and ''cat'', it does not matter that setting the ``dog'' is corresponding to the label `0' and the ``cat'' is `1' or inversely.

\begin{table*}[!ht]
    \centering
    \caption{Experimental Configurations}
    \label{tab:conf}
    \begin{tabular}{|l|l|l|l|l|}
    \hline
       \textbf{Dataset}  & MNIST & CIFAR10 & EMNIST & CIFAR100 \\ \hline
        \textbf{Target Model}& LeNet-5 & 6 conv + 2 fc & ResNet18 & ResNet50 \\ \hline
         \textbf{Split Point}  & 3nd cnov & 4th conv & 4th ResBlock & 6th ResBlock \\ \hline
    \end{tabular}
\end{table*}

%% file: eval.tex
\vspace{-2mm}
\section{Experimental Evaluation\label{eval}}
\begin{figure*}[htbp]
\centering
\subfigure[MNIST]{
\includegraphics[width=8cm]{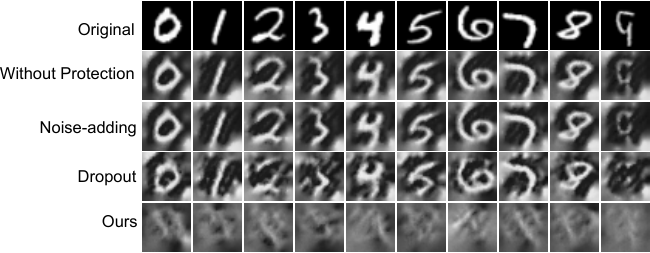}
}
\subfigure[CIFAR10]{
\includegraphics[width=8cm]{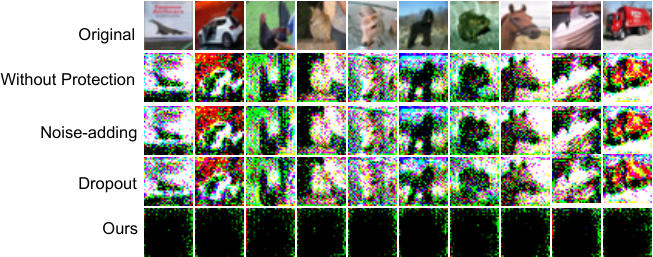}
\vspace{-1.0em}
}
\vspace{-1.0em}
\caption{Examples of visual comparison of images inversed by different methods.}
\label{fig:inver_comp}
\end{figure*}

\begin{table}
\centering
\caption{Quantitative evaluations for image recovery results. MSE and SSIM are measured on 100 testing samples.}
\label{tab:MSE_SSIM}
\begin{tabular}{c|l|c|c|c}
\hline
Dataset  & Method       & Accuracy         & MSE             & SSIM            \\ \hline
         & Baseline     & 99.05\%          & 0.7062          & 0.2742          \\ \cline{2-5} 
         & Noise-adding & 9.34\%           & 0.7175          & 0.0670          \\ \cline{2-5} 
MNIST    & Dropout      & 72.19\%          & 0.7751          & 0.0267          \\ \cline{2-5}
         & Pure local   & \textbf{96.71\%} & $\times$        & $\times$          \\ \cline{2-5} 
         & Ours         & 94.32\%          & \textbf{0.8942} & \textbf{0.0113} \\ \hline
         & Baseline     & 80.40\%          & 0.7532          & 0.4295          \\ \cline{2-5} 
CIFAR10  & Noise-adding & 11.32\%          & 0.7625          & 0.1336           \\ \cline{2-5} 
         & Dropout      & 37.25\%          & 0.8243          & 0.1132         \\ \cline{2-5} 
         & Pure local   & 32.17\%          & $\times$        & $\times$          \\ \cline{2-5}
         & Ours         & \textbf{71.46\%} & \textbf{0.9455} & \textbf{0.1025} \\ \hline
         & Baseline     & 89.19\%          & 0.7045          & 0.2644          \\ \cline{2-5} 
EMNIST   & Noise-adding & 14.35\%          & 0.7412          & 0.0217          \\ \cline{2-5} 
         & Dropout      & 47.22\%          & 0.7956          & 0.0379          \\ \cline{2-5}
         & Pure local   & 31.19\%          & $\times$        & $\times$          \\ \cline{2-5}
         & Ours         & \textbf{82.84\%} & \textbf{0.9546} & \textbf{0.0113} \\ \hline
         & Baseline     & 65.41\%          & 0.7245          & 0.4187          \\ \cline{2-5} 
CIFAR100 & Noise-adding & 9.46\%           & 0.7741          & 0.1787          \\ \cline{2-5} 
         & Dropout      & 26.56\%          & 0.8345          & 0.1154          \\ \cline{2-5} 
         & Pure local   & 6.21\%          & $\times$        & $\times$          \\ \cline{2-5}
         & Ours         & \textbf{57.46\%} & \textbf{0.9287} & \textbf{0.1125} \\ \hline
\end{tabular}
\vspace{-0.8cm}
\end{table}

\subsection{Experiment Setting}
We evaluate the proposed framework on four
standard DNN benchmark datasets: MNIST \cite{MNIST}, CIFAR-10 \cite{CIFAR10}, EMNIST \cite{EMNIST}, and CIFAR-100 \cite{CIFAR100}. MNIST is for handwritten digit recognition and CIFAR-10 is for object classification. EMNIST is an extension of MNIST, which contains 47 classes of handwritten digits and letters. CIFAR-100 is an extension of CIFAR-10, which contains 100 classes of objects. In addition, SVHN \cite{svhn_paper} is used to verify the effectiveness against shadow model attacks when the private dataset with a different domain from the original network training dataset,i.e., we assume the server trains the model on MNIST and a user holds the SVHN dataset. The SVHN dataset is a real-world dataset, taken from Google Street View imagery of house numbers, in a similar style to MNIST. We adopt LeNet \cite{lenet} on the MNIST dataset, and a CNN with 6 convolutional layers and 2 fully connected layers on the CIFAR-10 dataset. Further, we apply deeper models, ResNet18 and ResNet50 \cite{he2016deep}, for EMNIST and CIFAR100, respectively. It is worth noting that a user only owns 1/10 or nearly 1/10 of the total number of the standard datasets, we will introduce the specific setting for each experiment later. Roulette mainly focuses on privacy and accuracy issues, thus we run a CPU/GPU-based numerical emulation instead of deployment in a realistic network environment with hardware implementation. The runtime performance of co-inference deployed on the real-world system has been widely explored and optimized in many existing works, e.g., \cite{CoEdge2, CoEdge3}. To improve real-time performance, Roulette can apply the methods directly. In addition, we traverse all convolution layers as partition points to show the retrained accuracy, local model computation (FLOPs), and tensor size of the intermediate results in Section 5.4, which can be the reference for deployment. In our setting, after partitioning, although the part placed in the device can optimize and change the network structure according to the capacity of the device and communication cost, we first adopt the suggested partition point in \cite{Shredder} and keep the local model structure of the device partition for simplicity. We will then evaluate the impact of the partition point and local model structure. The target models and all attacks are implemented with Pytorch 1.8.1.  We run our experiments on a server with 1 NVIDIA Tesla A100 40G MIG 1g.5gb (the smallest MIG compute instance), and 2 * Intel(R) Xeon(R) Gold 6240 CPU @ 2.60GHz (A total of 36 cores/72 threads).

\begin{figure*}[htbp]
\centering
\subfigure{
\includegraphics[width=4.1cm]{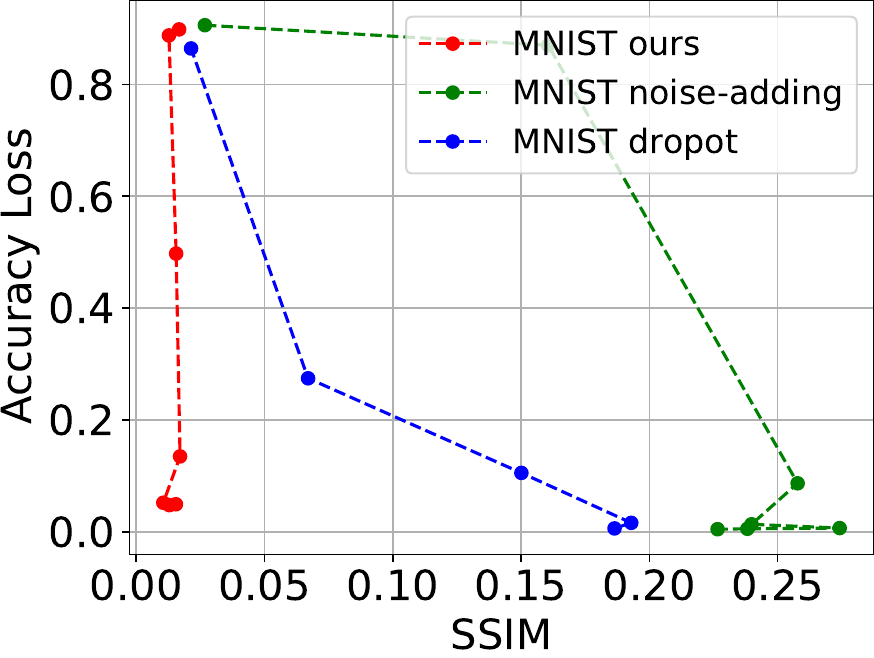}
}
\subfigure{
\includegraphics[width=4.2cm]{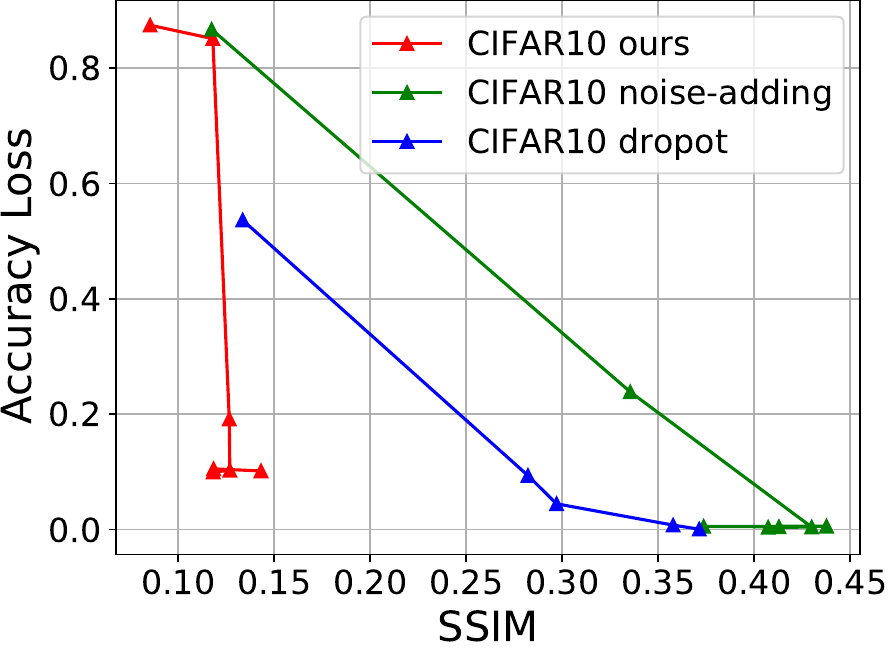}
}
\subfigure{
\includegraphics[width=4.2cm]{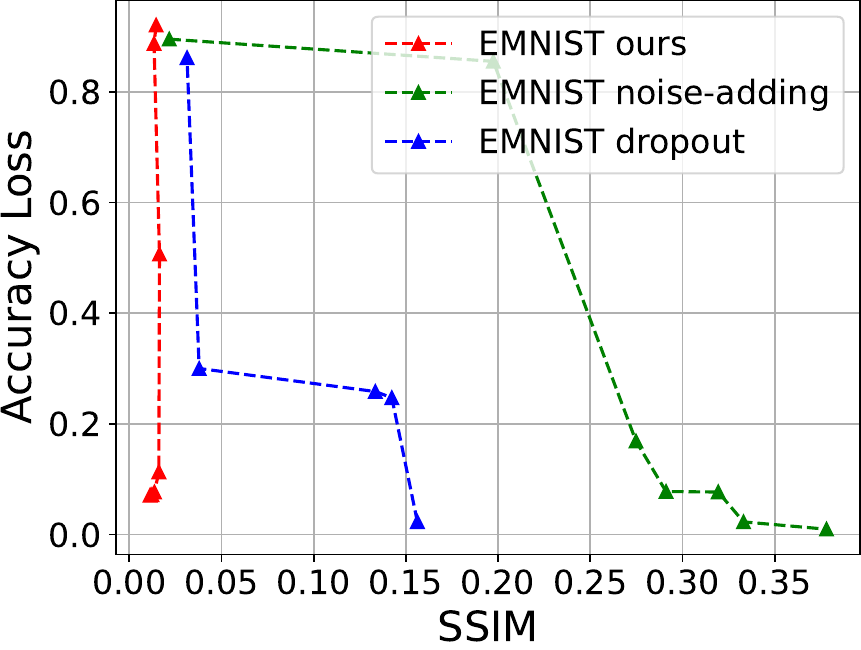}
}
\subfigure{
\includegraphics[width=4.1cm]{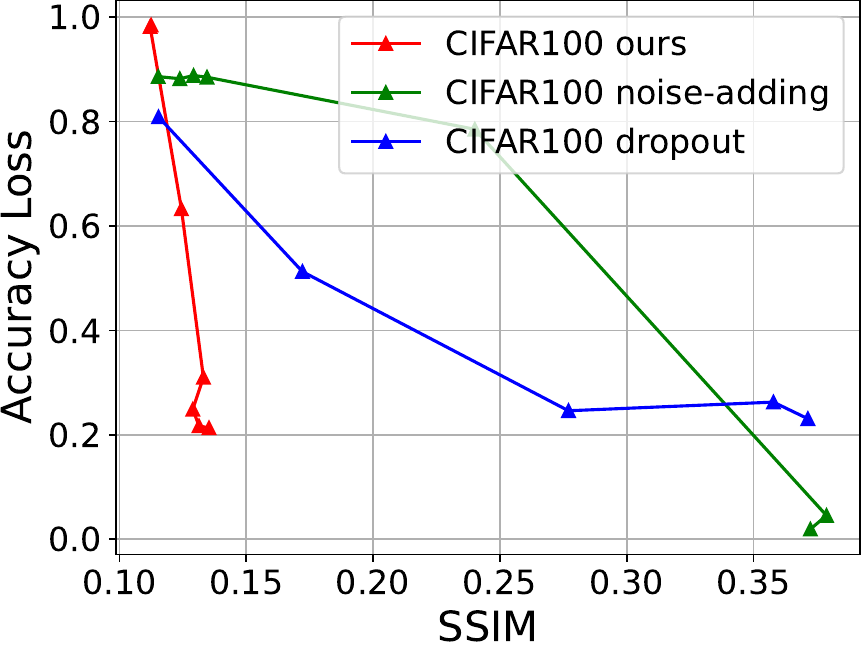}
\vspace{-1.0em}
}
\vspace{-1.0em}
\caption{Model accuracy loss versus the SSIM of inversed images on benchmark datasets.}
\label{fig:acc-metric}
\end{figure*}

\subsection{Effectiveness of Privacy Protection}
\subsubsection{Model Inversion Attack}
We evaluate the protection ability of the proposed framework against model inversion attack. Model inversion attack, namely, instance recovery attack is proposed in \cite{InversionAttack1}. Given a representation $z = f^{*}(x)$, the attacker tries to find the original input $x$ by solving:
\begin{equation}
    \label{eq:inversion}
    x^{*}=\arg \min_{s} \mathcal{L}(f^{*}(s),z),
    \vspace{-0.2em}
\end{equation}
where $\mathcal{L}$ is the loss function that measures the similarity between $f^{*}(s)$ and $z$, and (\ref{eq:inversion}) can be solved by iterative gradient descent. Note that when analyzing the model inversion attack, we assume that the dataset owned by device and edge server are consistent because the attacker is strongest in this setting \cite{InversionAttack1}. We evaluate the quality of inversion qualitatively and quantitatively through model inversion attack. In Fig. \ref{fig:inver_comp}, we show from a visual perspective the protection against the model inversion attack for some example testing images of the MNIST and CIFAR10 dataset. We show the original images and the reconstructed images derived by performing the attack to intermediate representation. Specifically, images in the first row are the original images. Images in the second and third row are recovered images with pure noise-adding \cite{canDP} and dropout \cite{InversionAttack2} defense, respectively. The fourth rows are the images recovered from the intermediate representation protected by the proposed framework. The proposed framework adds same scale of Laplace noise as the noise-adding method, i.e., $\epsilon=20$ for MNIST and EMNIST, and $\epsilon=100$ for CIFAR10 and CIFAR100, and we fix $\eta=0$ first. The dropout proportion is $10\%$ for MNIST and $5\%$ for CIFAR10. It can be seen in Fig. \ref{fig:inver_comp} that the recovered images with the noise-adding and dropout method are still visually recognizable but it is hard to get any useful information from the images under our proposed framework.

Further, we quantitatively measure the inversion performance by reporting the averaged similarity between 100 pairs of recovered images by the (\ref{eq:inversion}) and their original samples. The metrics we use to measure the similarity between the recovered image and the original image are Mean Squared Error (MSE) and structural similarity index metric (SSIM) \cite{SSIM}. Larger MSE and SSIM values indicate lower and higher similarity between the two images. We summarize the quantitative results under acceptable similarity (almost visually indistinguishable) in Table \ref{tab:MSE_SSIM}. It can be seen that under the acceptable similarity, the accuracy elimination of the noise and dropout methods is too large. Further, our framework integrates both noising-adding and dropout but achieves higher accuracy and lower similarity. This result indicates the superiority of Roulette compared to the naive noise-adding and dropout method. Because of personalized retraining, even integrating both noise adding and dropout for privacy preservation, Roulette achieves higher accuracy. Further, Roulette changes the label mapping relationship, which makes it attacker hard to infer information from the intermediate representation. Also, we consider the training and inference in a pure local manner with local models which have almost the same number of parameters as the front-end model. It turns out that Roulette achieves better model accuracy than the pure local method. 

Then, we plot the comparison between our framework, noise-adding, and dropout method with the accuracy loss as the vertical axis and the similarity index as the horizontal axis under varying the scale of the added noise ($\epsilon=1,5,10,20,50,100,200$) and dropout proportion ($\eta=0.05, 0.10, 0.15, 0.20$), which is shown in Fig. \ref{fig:acc-metric}. It can be seen that the pure noise-adding and pure dropout method without retraining will pay a huge price of privacy sacrifice when pursuing smaller model accuracy loss, but the proposed framework is insensitive to the noise scale defending against model inversion attack.

Both qualitative and quantitative results indicate that the proposed framework has a good defense against model inversion attack while taking into account that the accuracy loss is not too large.

\subsubsection{Gradient Inversion Attack}
Then, we try to perform the SOTA gradient-based reconstruction attack against federated learning, i.e., GradInversion \cite{GradInversion} on Roulette. Specifically, we use the same model and partition strategy in the previous and launch the attack on the gradients of the feature map. We assume a very strong attacker model which can get access to the exact model initial parameters of the local model. Also, we set the batch size is 1 and the attack is launched in the first round, which is the strongest attacker setting \cite{infocom23inver}. Even under this strong threat model, it turns out that Roulette defends the gradient inversion attack well. Two recovered examples are shown as follows in Fig. \ref{fig:g_re}.

\begin{figure}[ht]
    \centering
    \includegraphics{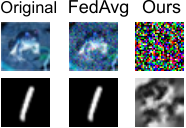}
    \caption{Reconstruction results of GradInversion}
    \label{fig:g_re}
\end{figure}

Indeed, attacks aiming at shared parameters (gradient) against federated learning can not reasonably be expected to have an effect on Roulette. Recall that Roulette divides the network into two parts, i.e., the local part (front-end) and the pre-trained remote part (back-end).  During both the training and inference, the things that users interact with the server are the intermediate representation (feature map), the gradients of the feature map, and the loss value. Roulette does not need to upload any local part parameters to the server. Also, the pre-trained remote part is trained using public datasets and is always unchanged as a deterministic function, which means the remote part parameters do not carry any information on local data. For the feature map, we have performed the reconstruction attack (model reversion attack) and showed that Roulette defenses the attack well in Section 5.2-Fig 4. The gradients of the feature map, compared to the feature map, go through much more non-linear computation and lost more information on the local data. The more non-linear computation, the more non-linear computation, the greater the difficulty of the inversion attack \cite{InversionAttack1}. Roulette also adds noise and rebuilds the data-label correspondence, making it harder for the reconstruction attack to be effective. 

\subsubsection{Shadow Model Attack}
In this subsection, we evaluate the more direct attack aiming at ground truth privacy. Targeting ground truth privacy, we propose shadow model attack which is to train a substitute model to imitate the real model and then perform attacks based on shadow model, e.g., inversion attack, or property inference attack. Specifically, the attacker trains a set of possible model $\mathcal{F^\prime}=\{f_{1}^{\prime},f_{2}^{\prime},...,f_{|\mathcal{S}|}^{\prime} \}$ corresponding to $|\mathcal{S}|$ possible label mappings and then, based on gradient descent, trains a discrimination network $D_\gamma$ by solving:
\begin{equation}
\gamma^{*}=\arg \min_{\gamma} \sum_{j}^{|\mathcal{S}|} \sum_{i}^{M} \mathcal{L}_{CE}(D_\gamma(f_{j}^\prime(x_i)), j),
\end{equation}
where $\mathcal{L}_{CE}$ is the cross entropy loss. Then, $D_\gamma$ is assigned to judge which mapping the device is using. We have analyzed that when $N$ is relatively large, shadow model attack is untraceable in the previous section. In this section, the worst-case empirical analysis is given. We evaluate the performance of the proposed system against shadow model attack on three-class classification system such that it is a binary classification problem for the attacker. We do not redesign a new classification system but re-label the original 10-class datasets to maintain the same feature extraction ability \cite{Shredder}. Specifically, for MNIST, we relabel the data as "$y\leq3$", "$3< y\leq 6$" and "$y>6$". For CIFAR10, we relabel the data as "mammal" (cat, dog, horse, deer), "non-mammal" (bird, frog), and "vehicle" (airplane, automobile, ship, truck). 

We have theoretically shown that an attacker is hard (exponential difficulty level) to launch a discrimination attack when the number of classes goes slightly. This subsection is a nearly worst-case analysis, e.g., three-class classification problems (We re-label the MNIST and CIFAR-10 as number range classification problem and mammal/non-mammal/vehicle classification problem) for the attacker without knowledge of model parameters but with data distribution. Further, we set a very strong attacker that if it figures out whether the user uses a specific key can be regarded as a successful attack. Then, we vary the distribution difference (non-i.i.d. degree) between the user's and the server's datasets to evaluate the attack performance. Preserving privacy thoroughly  without additional overhead is difficult, thus the experiment in this subsection is designed to explore the edge of the defense ability Roulette can provide. 

There are two major non-i.i.d issues according to \cite{non-iid}, e.g. label skew and attribute distribution skew. The first one means that the label distribution between the server and the user is different. The second one indicates that under the same label, the i.i.d. datasets share a similar style to each sample, and the non-i.i.d datasets are not and may introduce more interference, noise, and rotation. Hence, in this subsection, we consider three different local data distributions, which are almost the same as the original training dataset, non-iid label datasets, and datasets belonging to different domains (non-iid feature datasets). We define $\alpha$ to characterize the degree of data heterogeneity in the non-iid label data. Regard the device and the edge server as two agents. We uniformly randomly allocate samples with label $y$ to both agents whose last digits of their indices are $y$. For each sample allocated to agents, with a probability of $1-\alpha$, we allocate this sample to a random agent with equal probabilities, and with a probability of $\alpha$, the sample is allocated to the edge server; a larger $\alpha$ leads to greater data heterogeneity. The two datasets are i.i.d. if and only if $\alpha=0$.

First, we set $\epsilon=15,\ \eta=0.1$ for MNIST, $\epsilon=120,\ \eta=0.05$ for CIFAR10 and vary $\alpha$ ($0.2, 0.5, 0.8$). The evaluation results are shown in Fig. \ref{fig:LabHe}. The original accuracy refers to the test accuracy of the local dataset without retraining, new accuracy is the test accuracy with retrained local partition based on local data, and attack accuracy $\Lambda$ is defined as:
\begin{equation}
    \Lambda= 1- \frac{1}{M^{\prime}}\sum_{i}^{M^{\prime}}[ idx(\phi)\oplus D_\gamma(x_i)],
\end{equation}
where $idx(\phi)\in\{0,1,..., |\mathcal{S}|\}$ is the index of the true mapping corresponding to the shadow model in $\mathcal{F}^{\prime}$, $\oplus$ is XOR operation, and $M^{\prime}$ is test dataset size. It can be seen from Fig. \ref{fig:LabHe}, when the local data distribution is different from the original training dataset, our proposed framework can not only protect privacy, but also improve the inference accuracy. Meanwhile, the higher the degree of non-iid, the better the privacy protection effect of the proposed framework. When $\alpha = 0.8$, the probability of successful attack is almost equal to random guess. On the other hand, if the degree of non-iid between two agents' dataset is low, the proposed method has a high risk of ground truth privacy disclosure. Note that for the attack accuracy, random guess value is 0.5, and for the original and retained accuracy, random guess value is nearly 0.33.
\begin{figure}[htbp]
\centering
\subfigure[MNIST]{
\includegraphics[width=4cm]{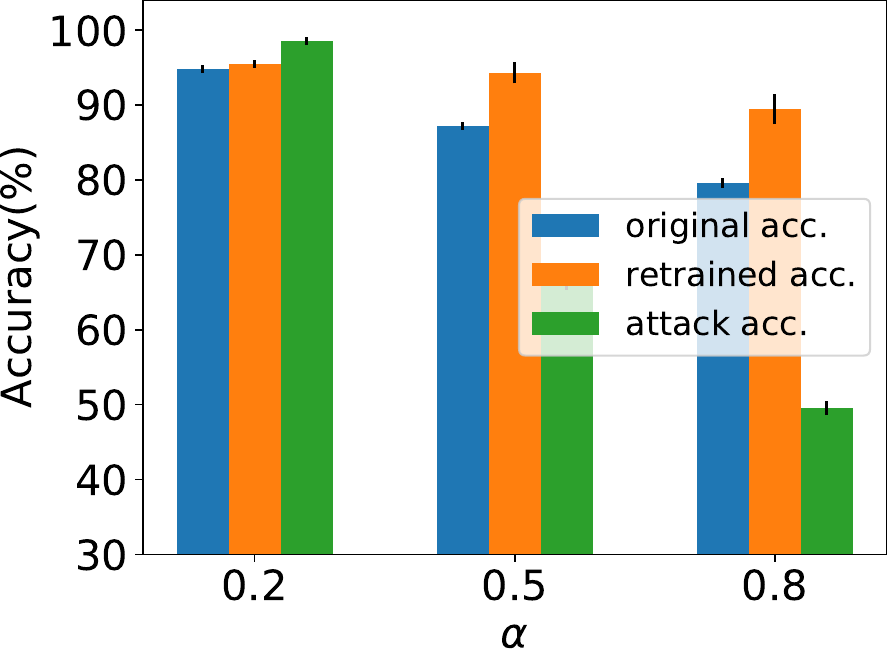}
}
\subfigure[CIFAR10]{
\includegraphics[width=4cm]{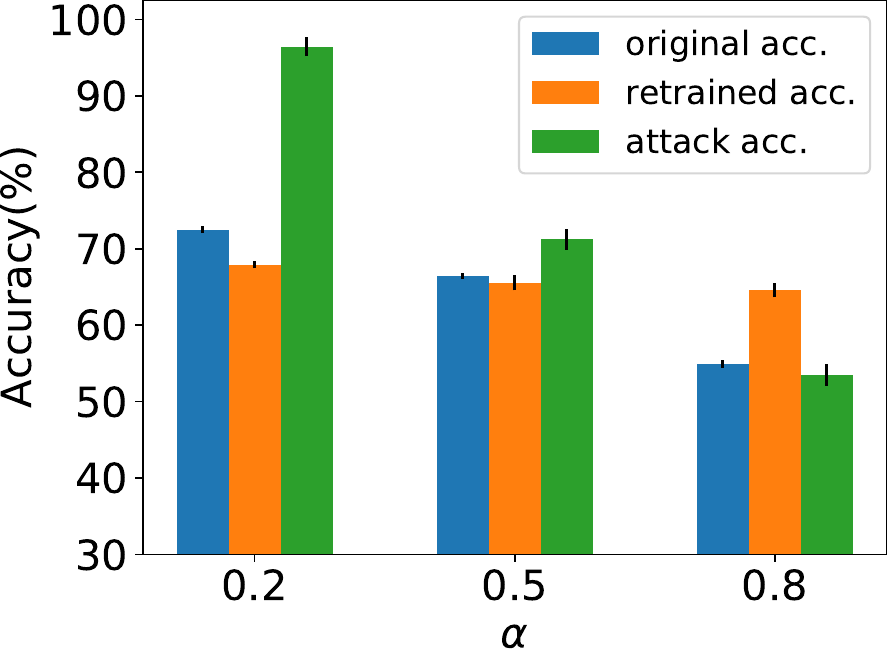}
}
\vspace{-1.0em}
\caption{Original, retrained and attack accuracy versus $\alpha$ on MNIST (a) and CIFAR10 (b).}
\label{fig:LabHe}
\vspace{-1.0em}
\end{figure}

Further, we evaluate the performance of the proposed framework when the domain of the local dataset is different from the original training dataset. Specifically, the original dataset is MNIST, and we retrain the local partition on the SVHN. To simulate the limited data volume of the device, we set the retrain dataset to be uniformly sampled at 1/10 of the entire SVHN. The visual comparison of data samples between MNIST and SVHN is shown in Fig. \ref{fig:MSComp}. It can be seen that the feature distribution of MNIST and SVHN is different but both datasets are digit numbers. In the experiment, the original RGB images in SVHN are converted into gray images. The average original, retrained and attack accuracy are 0.4428, 0.7683, and 0.4922, respectively, which means that when the domain difference is large, the model accuracy eliminates seriously, and both privacy and model accuracy are improved via the proposed framework. 
\begin{figure}[htbp]
\centering
\subfigure[MNIST]{
\includegraphics[width=4cm]{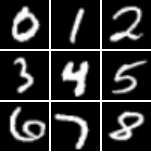}
}
\subfigure[SVHN]{
\includegraphics[width=4cm]{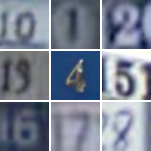}
}
\vspace{-1.0em}
\caption{Comparison of MNIST (a) and SVHN (b) dataset.}
\label{fig:MSComp}
\vspace{-1.0em}
\end{figure}

\begin{figure}[htbp]
\centering
\subfigure[Attack acc. (MNIST)]{
\includegraphics[width=4cm]{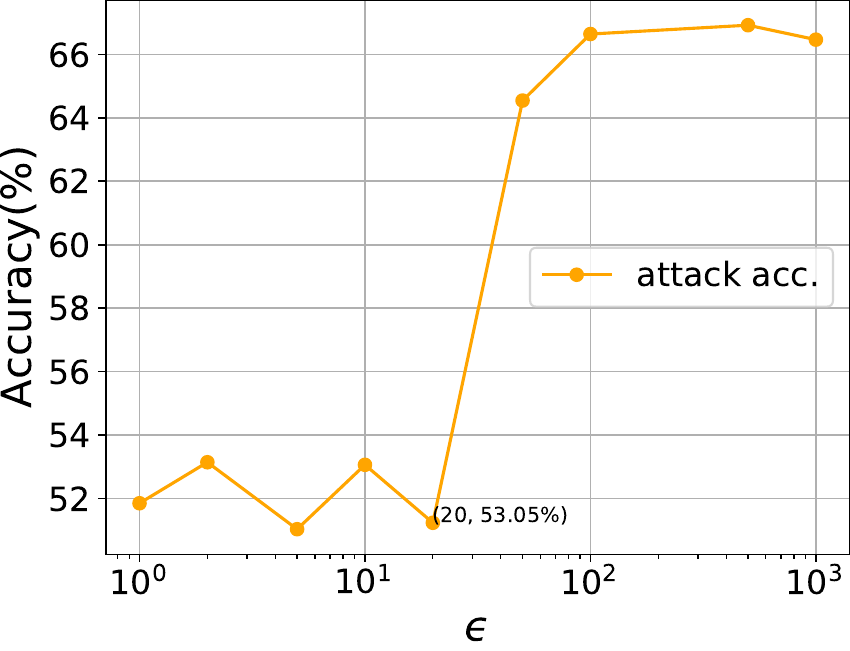}
}
\subfigure[Retrained acc. (MNIST)]{
\includegraphics[width=4cm]{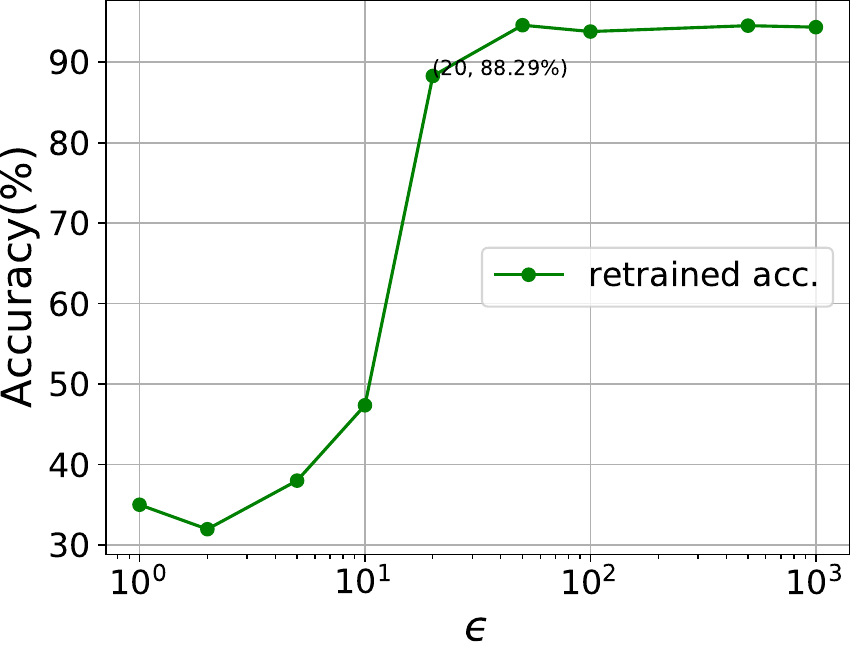}
}
\subfigure[Attack acc. (CIFAR10)]{
\includegraphics[width=4cm]{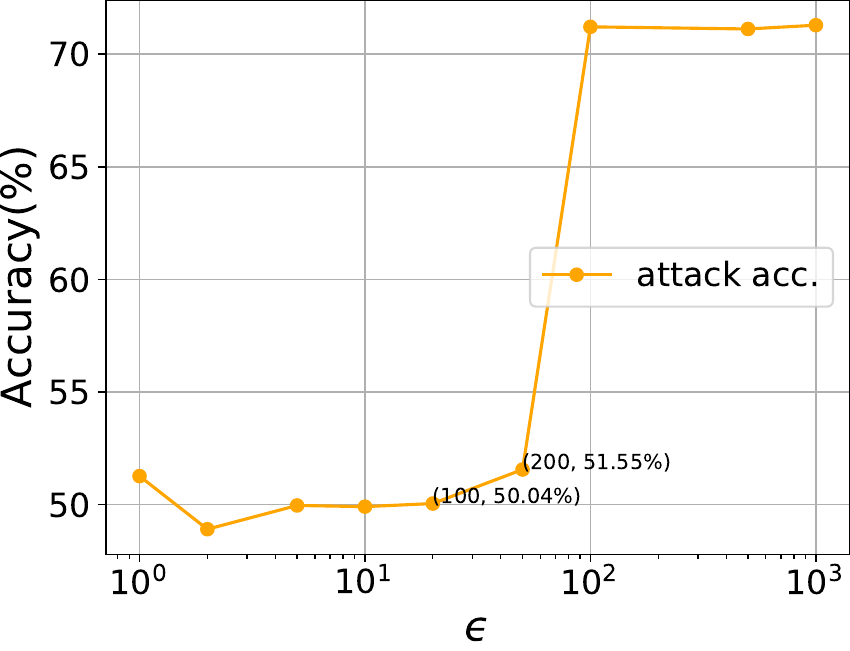}
}
\subfigure[Retrained acc. (CIFAR10)]{
\includegraphics[width=4cm]{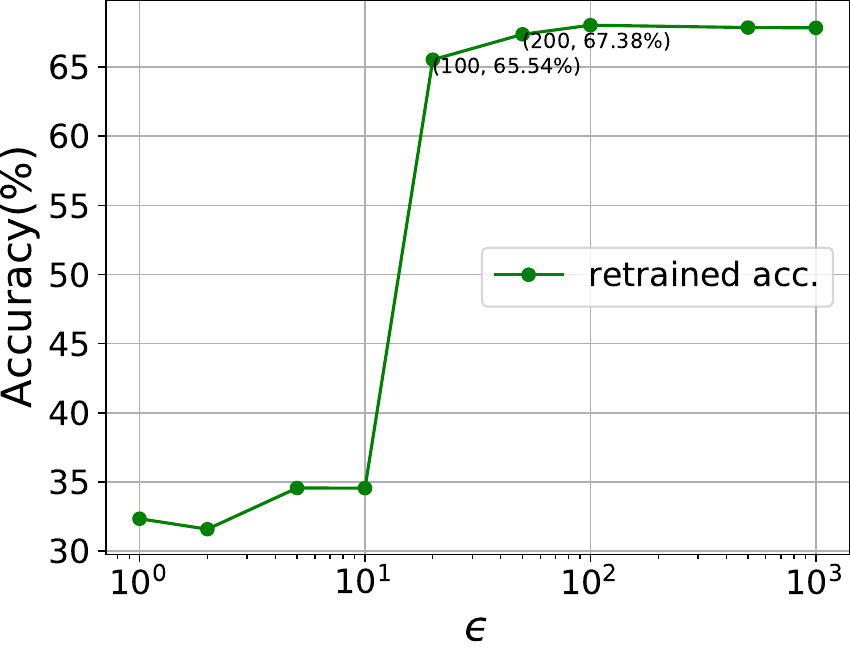}
}
\vspace{-1.0em}
\caption{Attack and retrained accuracy versus $\epsilon$.}
\label{fig:epsilon}
\vspace{-1.0em}
\end{figure}

\subsection{Impact of Differential Privacy  Parameters}
We first analyze the impact of differential privacy budget parameter $\epsilon$. In this section, we still take the three-class classification systems in the last section as the research object. We fix $\alpha = 0.5,\ \eta=0$ and vary $\epsilon$ from 1 to 1000 to evaluate the shadow model attack accuracy and retrained accuracy. As depicted in Fig. \ref{fig:epsilon}, for MNIST, the DNN under the Laplace mechanism has essentially no utility for $\epsilon < 20 $. However, when $\epsilon\geq 20$, the accuracy is rapidly getting close to the noise-free accuracy. For CIFAR10, the threshold is about 100. Then, it can be seen that for MNIST, when $\epsilon = 20$, the retraining accuracy reaches 0.8829, while the attack accuracy is only 0.5123. For CIFAR10, when $\epsilon = (100,200)$, the retraining accuracy reaches 0.6554, 0.6738, while the attack accuracy is only 0.5004, 0.5155. This shows that the proposed framework has sweet spots of noise scale choosing where the attack can be defended meanwhile model utility is maintained. We also notice that even the attack accuracy becomes big and stable when $\epsilon$ greater than the threshold, but the value is still relatively small (only 10\%+ than random guess). However, generally, high values of $\epsilon$ ($>20$) may still provide meaningless differential privacy \cite{IsDP}, which means that original Laplace mechanism is not suitable to directly use but we have the Theorem \ref{thm:DPbudget} to reduce the value.

Then, we analyze the total differential privacy budget based on the Theorem \ref{thm:DPbudget}, which is controlled by two privacy related parameters, i.e, $\epsilon$ and $\eta$. Here, we fix the value of $\eta$ as 0.1 for MNIST and 0.05 for CIFAR10, and change the privacy budget $\varepsilon$ by varying $\epsilon$ 1 to 1000. We set the noise-adding layer is after the first conv. layer. Fig. \ref{TotBudget}(a) shows that the proposed framework can maintains the accuracy at a high value for a wide range of privacy budgets. The results also shows that the nullification operation plays an important role. Compared to $\epsilon$, $\varepsilon$ is obliviously small. These results argue that the proposed framework is applicable to different privacy requirements. Member inference attacks can be defended when privacy budgets are tight \cite{membershipdf}. 

The result also indicate that it can effectively improve the performance even when the privacy budget is much tighter compared to the method using noise-adding and dropout under same hyper-parameters but without retraining, which is shown in Fig. \ref{TotBudget} (b). It is obvious that the threshold budget of getting close to original accuracy for the naive noise-adding is larger than the proposed method. This improvement is brought by local retraining.

\begin{figure}[htbp]
\centering
\subfigure[Our framework]{
\includegraphics[width=4cm]{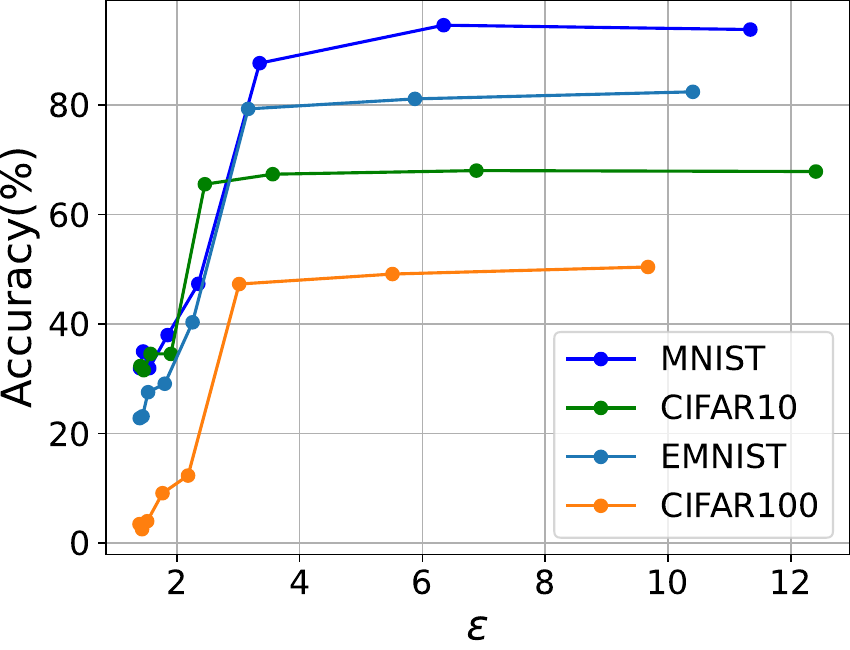}
}
\subfigure[Without Retraining]{
\includegraphics[width=4cm]{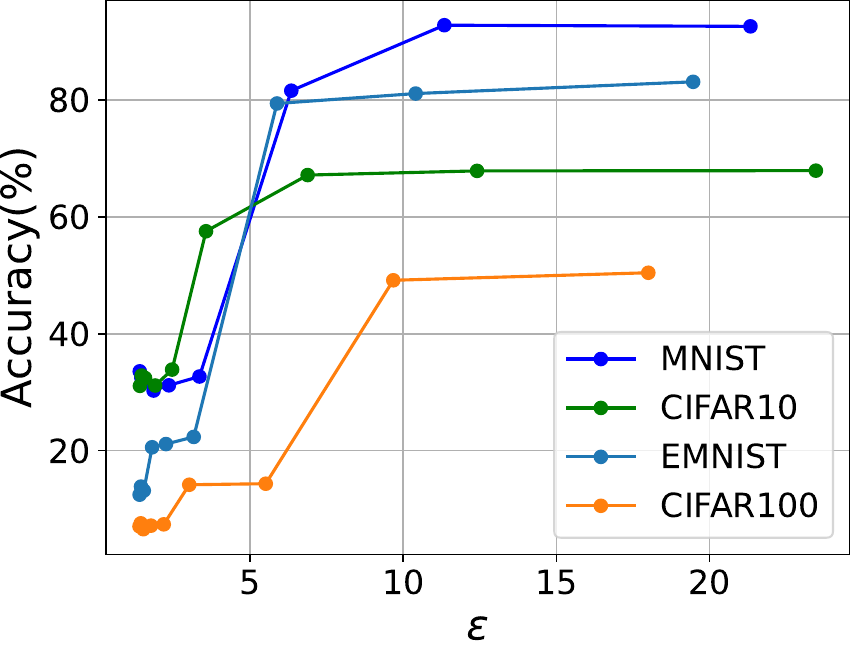}
}
\vspace{-1.0em}
\caption{Model accuracy versus $\varepsilon$.}
\label{TotBudget}
\vspace{-1.0em}
\end{figure}

\subsection{Impact of Local Model Structure and Partition Point}
The proposed framework allows changing the local model structure at the device, which can avoid the problem of early partitioning when the partition point is at the early layer with insufficient model capacity that eliminates the inference accuracy. However, the price of using a larger structure is that the local device with limited resources needs to undertake more computing, resulting in more computing latency. Therefore, in this subsection, we explore how to choose a partition point and design local network structure empirically. To do this, we traverse the potential partition points of the model and increase the number of layers of the local model to observe the inference accuracy. The computing loads of different local models and intermediate tensor size are shown in Table \ref{tab:3}, which are the indicators of local computing delay and communication delay with the edge server, respectively. Note that in order to keep the output shape of the local model unchanged, the convolution kernel size of the layers we added is $3\times3$ and the padding is 1 in the experiment. 

We use approximate thresholds to delineate an acceptable range in Table III. For LeNet, we regard FLOPs less than 40k, accuracy greater than 0.9, and tensor size less than 500 as acceptable values. For CIFAR10, we regard FLOPs less than 100k, accuracy greater than 0.6, and tensor size less than 20k as acceptable values. For EMNIST, we regard FLOPs less than 3M, accuracy greater than 0.9, and tensor size less than 10k as acceptable values. For CIFAR100, we regard FLOPs less than 1.5M, accuracy greater than 0.45, and tensor size less than 10k as acceptable values. Note that the approximate thresholds are empirically selected. Specific selection can refer to the method in \cite{CoEdge2}. The split point corresponding to the row with all three values grayed out is recommended. It can be seen that when the partition point is too early, the size of the intermediate representation is often large, and adding new layers is very inefficient for accuracy improvement. In general, choosing deeper partition points and modifying fewer structures are more effective than the opposite. In practice, we can choose a proper partition point that is deep enough but also matches the computing capability of the hosted devices.

\begin{table}
\centering
\caption{Comparison of different local structures: cells marked gray mean that the data is acceptable.}
\vspace{-1.0em}
\subtable[LeNet for MNIST]{
\begin{tabular}{cccc}
\cline{1-4}
\multicolumn{1}{c|}{Partition Point}      & \multicolumn{1}{c|}{FLOPs (k)}                       & \multicolumn{1}{c|}{Accuracy}                       & Tensor Size             \\ \cline{1-4}
\multicolumn{1}{c|}{}      & \multicolumn{1}{c|}{\cellcolor[HTML]{EFEFEF}12.70} & \multicolumn{1}{c|}{0.4723}                         &                              \\
\multicolumn{1}{c|}{1conv} & \multicolumn{1}{c|}{\cellcolor[HTML]{EFEFEF}35.41} & \multicolumn{1}{c|}{0.7328}                         & 1176                          \\
\multicolumn{1}{c|}{}      & \multicolumn{1}{c|}{543.62}                         & \multicolumn{1}{c|}{\cellcolor[HTML]{EFEFEF}0.9578} &                              \\ \cline{1-4}
\multicolumn{1}{c|}{2conv} & \multicolumn{1}{c|}{\cellcolor[HTML]{EFEFEF}37.65} & \multicolumn{1}{c|}{\cellcolor[HTML]{EFEFEF}0.9425} & \cellcolor[HTML]{EFEFEF}400  \\ \cline{1-4}
\multicolumn{1}{c|}{3conv} & \multicolumn{1}{c|}{42.48}                         & \multicolumn{1}{c|}{\cellcolor[HTML]{EFEFEF}0.9622} & \cellcolor[HTML]{EFEFEF}120 \\ \cline{1-4}
\end{tabular}
}
\quad
\subtable[CIFAR10CNN]{
\begin{tabular}{c|c|c|c}
\hline
Partition Point & FLOPs (k)                     & Accuracy                       & Tensor size                   \\ \hline
                & \cellcolor[HTML]{EFEFEF} 1.9                           & 0.2548                         &                               \\
1 conv          & \cellcolor[HTML]{EFEFEF}39.78                         & 0.4753                          & 65536                         \\
                & \cellcolor[HTML]{EFEFEF}77.66                         & 0.5547                        &                               \\
                & 115.54                        & 0.5722                         &                               \\ \hline
                & 39.85                         & 0.4986                         & \cellcolor[HTML]{FFFFFF}      \\
2 conv          & \cellcolor[HTML]{EFEFEF}49.32 & \cellcolor[HTML]{EFEFEF}0.6423 & \cellcolor[HTML]{EFEFEF}16384 \\ \hline
3 conv          & \cellcolor[HTML]{EFEFEF}58.79 & \cellcolor[HTML]{EFEFEF}0.7016 & 32768                         \\ \hline
4 conv          & \cellcolor[HTML]{EFEFEF}96.63 & \cellcolor[HTML]{EFEFEF}0.7261 & \cellcolor[HTML]{EFEFEF}8192  \\ \hline
5 conv          & 106.09                        & \cellcolor[HTML]{EFEFEF}0.7584 & \cellcolor[HTML]{EFEFEF}8192  \\ \hline
6 conv          & 115.54                        & \cellcolor[HTML]{EFEFEF}0.7982 & \cellcolor[HTML]{EFEFEF}2048  \\ \hline
\end{tabular}
}
\subtable[ResNet18 for EMNIST]{
\begin{tabular}{c|c|c|c}
\hline
\rowcolor[HTML]{FFFFFF} 
Partition point                    & FLOPs (k)                      & Accuracy                 & Tensor size                  \\ \hline
\rowcolor[HTML]{FFFFFF} 
2 ResBlock                         & \cellcolor[HTML]{EFEFEF}149.82 & 0.6345                   & 65536                        \\ \hline
\rowcolor[HTML]{EFEFEF} 
\cellcolor[HTML]{FFFFFF}4 ResBlock & 675.39                         & 0.8284                   & 32768                        \\ \hline
\rowcolor[HTML]{EFEFEF} 
\cellcolor[HTML]{FFFFFF}6 ResBlock & 2775.10                        & 0.8842                   & 16384                        \\ \hline
\rowcolor[HTML]{FFFFFF} 
8 ResBlock                         & 11168.83                       & 0.8991                   & \cellcolor[HTML]{EFEFEF}8192 \\ \hline
\end{tabular}
}
\subtable[ResNet50 for CIFAR100]{
\begin{tabular}{c|c|c|c}
\hline
Partition point & FLOPs(k)                         & Accuracy                       & Tensor size                  \\ \hline
ResBlock1       & \cellcolor[HTML]{EFEFEF}84.54   & 0.0456                         & 16384                        \\ \hline
ResBlock2       & \cellcolor[HTML]{EFEFEF}154.94   & 0.0873                         & 16384                        \\ \hline
ResBlock3       & \cellcolor[HTML]{EFEFEF}225.34  & 0.1963                         & 16384                        \\ \hline
ResBlock4       & \cellcolor[HTML]{EFEFEF}604.74  & \cellcolor[HTML]{EFEFEF}0.4599 & \cellcolor[HTML]{EFEFEF}8192 \\ \hline
ResBlock5       & \cellcolor[HTML]{EFEFEF}884.80    & \cellcolor[HTML]{EFEFEF}0.4631 & \cellcolor[HTML]{EFEFEF}8192 \\ \hline
ResBlock6       & \cellcolor[HTML]{EFEFEF}1164.86 & \cellcolor[HTML]{EFEFEF}0.5746 & \cellcolor[HTML]{EFEFEF}8192 \\ \hline
ResBlock7       & \cellcolor[HTML]{EFEFEF}1444.93 & \cellcolor[HTML]{EFEFEF}0.5846 & \cellcolor[HTML]{EFEFEF}8192 \\ \hline
ResBlock8       & 2957.376                         & \cellcolor[HTML]{EFEFEF}0.6213 & \cellcolor[HTML]{EFEFEF}4096 \\ \hline
ResBlock10      & 5191.744                         & \cellcolor[HTML]{EFEFEF}0.6344 & \cellcolor[HTML]{EFEFEF}4096 \\ \hline
ResBlock12      & 7426.112                         & \cellcolor[HTML]{EFEFEF}0.6347 & \cellcolor[HTML]{EFEFEF}4096 \\ \hline
ResBlock13      & 8543.296                         & \cellcolor[HTML]{EFEFEF}0.6412 & \cellcolor[HTML]{EFEFEF}4096 \\ \hline
ResBlock12      & 14582.848                        & 0.6408                         & \cellcolor[HTML]{EFEFEF}2048 \\ \hline
\end{tabular}
}
\vspace{-0.5cm}
\label{tab:3}
\end{table}



%% file: rela.tex
\section{Related Works\label{rela}}

In this section, we review the related works in the field of co-inference and related privacy protection solutions. These privacy protection solutions can be divided into three main categories: noise injection, encoding-based solutions, and cryptographic solutions, and each category is reviewed in the following subsection.

\subsection{Collaborative Inference} As a prominent solution for fast edge inference, collaborative inference (co-inference) has been widely studied. The basic idea is to partition DNN into multiple partitions, with each part allocated to different participants. The partition system is designed to choose the optimal partition points according to computation, communication, and energy constraints in \cite{CoEdge3,CoEdge4}. More recently, to orchestrate the workload among a cluster of heterogeneous devices, CoEdge \cite{CoEdge} adaptively partitions the input images tailored to the diverse computing capacities and dynamic network conditions. Furthermore, Autodidactic Neurosurgeon \cite{CoEdge2} applies online learning to predict uncertain remote time. These works make their effort to improve real-time performance but He et al. \cite{InversionAttack2} point out that co-inference may lead to privacy leakage and suggests that the researchers in this area pay more attention to privacy issues instead of considering real-time performance only \cite{InversionAttack1,InversionAttack2}. And then, analyzing the relation between privacy leakage and depth of the DNN, a Lyapunov optimization framework is proposed to choose partition point considering both real-time performance and privacy \cite{LavPri}. However, it only measures the attenuation of information related to the primary task at different partition points without any privacy enhancement. 

\subsection{Semantic Communication} Semantic and task-oriented communication is a hot topic in the communication research community. The core idea of semantic or task-oriented communication is to build a next-generation communication system by considering the "meanings" of sent information at a transmitter instead of regarding each bit equal \cite{SC1,SC2}. In recent works, split learning and co-inference are widely applied in semantic communication systems. J. Shao et al. designed a task-oriented communication system, where the DNN deployed on the device plays a role of both feature extractor and joint source-channel coding encoder \cite{spl-co1}. Z. Weng et al. proposed semantic communication systems for speech \cite{spl-co2}. In \cite{spl-co2}, the transmitter and receiver are composed of semantic and channel encoder (decoder). Both encoder and decoder are in the form of DNN, and the training procedure follows the split learning mode. However, in these works, they do not take the privacy issues into account. For more details about semantic and task-oriented communications, we refer the reader to the surveys \cite{SC1} and \cite{SC2}.

\subsection{Noise Injection} Noise-based privacy protection is widely studied in data collection and model training but only a few involve inference. An obvious shortcoming of noise injection is that it always imposes an abstruse trade-off between privacy and utility. To find a better balance between utility and privacy, Mireshghallah et al. propose Shredder \cite{Shredder}, a device-edge co-inference system, in which noise added to the IR is drawn from learned noise distribution instead of random noise. Furthermore, based on Differential Privacy (DP), Cloak \cite{Cloak} treats the pre-trained network on the cloud as a deterministic function, and designs a Laplace mechanism that satisfies Local $\epsilon$-DP to add noise to raw data for further inference. However, the inference result of the primary mask is exposed to the cloud or edge server directly.

\subsection{Privacy-preserving Federated Learning} 
Federated Learning (FL) has emerged as a promising approach to enable collaborative machine learning on decentralized data sources while preserving data privacy \cite{Fed1}. FL leverages a decentralized architecture where multiple devices or nodes collectively train a shared model without sharing their raw local data. In recent years, several works leverage some features, or even the raw local data can be recovered from the gradients and model parameters during multi-party interaction \cite{FLprivacy1, FLprivacy2}. Privacy-preserving federated learning (PPFL) focuses on enhancing the privacy guarantees within the FL framework. It employs various techniques to protect the confidentiality of sensitive data during the training process \cite{SMCFL, HEFL}. FL with submodel design may be an alternative to Roulette, e.g., \cite{comment_1, comment_2}. However, it does not provide rigorous privacy and security analysis. Also, some submodel designs can not be deployed directly on the general classification problems. For instance, the system proposed in \cite{comment_1} can only be applied in a very special scenario where the recommendation system model for each client can be naturally decoupled.

\subsection{Encoding-based Solutions} 
Encoding-based solutions train a local encoder and remote decoder for the device and the server separately. These solutions can be divided into privacy-preserving latent representation learning and instance encoding. The most widely studied privacy-preserving latent representation learning is adversarial training in which given privacy attribute labels, they perform a min-max game between defender and adversary to learn a privacy-preserving representation \cite{DPFB,GANdefense1,GANdegense2}. Furthermore, Li et al. proposed TIPRDC \cite{IntermRep1}, a task-independent framework, in which the encoder can generate a privacy-preserving representation without the knowledge of privacy attribute labels. MixCon \cite{MixCon} study the relationship between model inversion attack and separability of representation and develop consistency loss adjusting the separability of the representation to balance data utility and hardness of inversion attack. These approaches can not protect the information of the primary task. Complex Valued Network \cite{ConplexedVluedNetwork} not only uses adversarial training but also transfers real numbers into complex numbers to hide private information which significantly reduces the attacker’s success rate. However, all of these approaches including TIPRDC, and Complex Valued Network, do not take efficiency into account when they consider privacy issues. Instance encoding, DataMix \cite{DataMix} and InstaHide \cite{Instahide}, encode several instances in a raw dataset to generate a new dataset and use it to train a new model which has almost the same accuracy as the original model trained on the raw dataset, which seems a 'free lunch', but recent study \cite{carliniprivate} shows that, 'there is no free lunch', the privacy guarantee provided by instance encoding can be defeated at low cost. 

\subsection{Cryptographic Solutions}
Cryptography-based inference mainly includes three methods: homomorphic encryption (HE) \cite{HE}, secure multi-party computing (SMC)\cite{SMC}, and trusted execution environment (TEE) \cite{TEE}. HE allows computation over encrypted data. There are some existing works dedicated to the application of HE and SMC for privacy-preserving inference, but they all have serious computation delays and require special additional hardware. Even as the state-of-the-art of HE, DELPHI \cite{HE3} slows down the inference speed by 318 $\times$ \cite{Cloak}, which is unacceptable in the AIoT scenario. Similar shortcomings also exist in TEE-based methods, which require server upgrading hardware as secure enclaves (e.g., \cite{TEE,Outsourcing}) and have been shown to be vulnerable to side-channel attacks.


%% file: conclusion.tex
\section{Conclusion\label{conclusion}}
In this paper, we propose Roulette, a semantic privacy-preserving framework for deep learning classifiers. Different from traditional co-inference, we retrain the local network based on split learning, so that the remote network cannot obtain correct inference results, thus protecting semantic privacy and boosting model accuracy. Further, in order to achieve differential privacy, local devices are allowed to do nullification on input data and add noise before uploading the DNN intermediate representation to the edge server. We give the privacy guarantee and analyze the hardness of the ground truth inference attack. We extensively evaluate the proposed framework, testing its performance on defense against various major attacks. Experimental results show that the proposed framework can effectively protect semantic privacy and meanwhile achieve good model accuracy. The limitation of this work is that we do not build a unified theoretical analysis framework for different background knowledge attackers, which will be considered in future work.


%% file: appendix.tex
\appendices
\section{proof of theorem \ref{thm:DPbudget}}
Before proving Theorem \ref{thm:DPbudget}, we first prove the following two theorems.
\begin{thm}
\label{thm:A.1}
Given a input $x$ and a deterministic function $f$, $\left| f(x) \right|\leq B$, $\forall a \in \mathbb{R}^{+}$, the random mechanism $\hat{f}(x)=f(x)+aLap(2B/\sigma)$ is $(\frac{\sigma}{a})$-differentially private.
\end{thm}
\begin{proof}
For any adjacent inputs $x$ and $x^{\prime}$,
\begin{equation}
    \begin{aligned}
        &\quad \frac{\Pr[(f(x)+aLap(2B/\sigma)=S]}{\Pr[(f_{\theta}(x^{\prime})+aLap(2B/\sigma)=S]}\\
        &=\frac{e^{-\frac{\left| S-f(x)\right|\sigma}{2aB}}}{e^{-\frac{\left| S-f(x^{\prime})\right|\sigma}{2aB}}}\\
        &=e^{\frac{\sigma}{2aB}(\left|S-f(x^{\prime})\right|-\left|S-f(x)\right|)} \\
        &\leq e^{\frac{\sigma}{2aB} \left| f(x)-f(x^{\prime})\right|} \\
        &\leq e^{\frac{\sigma}{2Ba} \cdot 2B} = e^{\frac{\sigma}{a}}.
    \end{aligned}
\end{equation}
Based on the Definition \ref{def:1}, the $\hat{f}(x)$ is $(\frac{\sigma}{a})$-differentially privacy.
\end{proof}

\begin{thm}
\label{thm:A.2}
Given a input $x\in\mathbb{R}^{p}$, and an arbitrary nullification matrix $I_p\stackrel{R}\gets\{0,1\}^p$ with nullification rate $\eta$, if $\hat{f}(x)$ is a $\epsilon$-differential private mechanism, then $\hat{f}(x\odot I_p)$ is a $\varepsilon$-differentially privacy mechanism where
\begin{equation}
    \varepsilon=\ln\left[(1-\eta)e^{\epsilon} +\eta \right].
\end{equation}
\end{thm}
\begin{proof}
Assuming there are two adjacent inputs $x_1$ and $x_2$ differing only one element $i$, i.e., $x_1=x_2\cup i$. We denote $x_1^{\prime}=x_1\odot I_p,\ x_2^{\prime}=x_2\odot I_p$. There are two possible cases after the operation, i.e., $i\in x_1^{\prime}$ and $i \not\in x_1^{\prime}$. For the first case, it is oblivious that $x_1\odot I_p=x_2\odot I_p$, because the only different element between $x_1$ and $x_2$ are both multiplied by 0. Hence, we have,
\begin{equation}
    \Pr [\hat{f}(x_1\odot I_p)=S] = \Pr [\hat{f}(x_2\odot I_p)=S].
\end{equation}
For the second case, since $\hat{f}$ is $\epsilon$-differentially private, according to Definition \ref{def:1} we have,
\begin{equation}
    \Pr [\hat{f}(x_1\odot I_p)=S] \leq \Pr e^{\epsilon}[\hat{f}(x_2\odot I_p)=S].
\end{equation}
Further, based on the fact that $\Pr[i\not\in x_1^{\prime}]=\eta$, we have,
\begin{equation}
    \begin{aligned}
        &\quad  \Pr [\hat{f}(x_1\odot I_p)=S]\\
        &=\eta\Pr [\hat{f}(x_1\odot I_p)=S]+(1-\eta)\Pr [\hat{f}(x_1\odot I_p)=S]\\
        &\leq \eta\Pr [\hat{f}(x_1\odot I_p)=S]+(1-\eta)e^{\epsilon}[\hat{f}(x_2\odot I_p)=S]\\
        &=\exp{[\ln{(1-\eta)e^{\epsilon}+\eta}]}\Pr [\hat{f}(x_1\odot I_p)=S].
    \end{aligned}
\end{equation}
Again, based on Definition \ref{def:1}, we can conclude that $\hat{f}(x\odot I_p)$ is a $\varepsilon$-differentially privacy mechanism where
\begin{equation}
    \varepsilon=\ln\left[(1-\eta)e^{\epsilon} +\eta \right].
\end{equation}
\end{proof}

Then, we prove Theorem \ref{thm:DPbudget}.
\begin{proof}
First of all, we analyze the output of the local net with the proposed Laplace mechanism but without nullification, i.e.,
\begin{equation}
\label{eq:Aouput}
    \begin{aligned}
        \hat{f}(x_l)&=f_\theta^{(2)}\left(f_\theta^{(1)}(x_l)+Lap\left(B/\epsilon\right)\right)\\
        &=f_\theta^{(2)}\left(x_r+Lap\left(B/\epsilon\right)\right).
    \end{aligned}
\end{equation}
Since, generally, $Lap(B/\epsilon)$ is much smaller than $x_r$, we can approximate (\ref{eq:Aouput}) by applying the first-order Taylor series:
\begin{equation}
\label{eq:AApproOut}
    \begin{aligned}
        &\quad f_\theta^{(2)}\left(x_r+Lap\left(B/\epsilon\right)\right)\\
        &\approx f_\theta^{(2)}(x_r)+(\nabla_{x_r} f_\theta^{(2)})^{\textsf{T}}Lap(B/\epsilon)\\
        &=f_\theta(x_l)+(\nabla_{x_r} f_\theta^{(2)})^{\textsf{T}}Lap(B/\epsilon).
    \end{aligned}
\end{equation}
Base on Theorem \ref{thm:A.1}, we can say \ref{eq:AApproOut} is $(\epsilon/\Xi)$-differentially private, and $\Xi=\lVert \nabla_{x_r}f_{\theta_t}^{(2)}\rVert_\infty$. Further, by applying the Theorem $\ref{thm:A.2}$, we have,
\begin{equation}
    \varepsilon=\ln\left[(1-\eta)e^{\epsilon/\Xi} +\eta \right].
\end{equation}
\end{proof}

\section{Proof of theorem \ref{thm:hardness}}
We provide a hardness of approximation result for finding the ground truth inference function. In particular, we prove unless \textsf{RP=NP}, there is no polynomial time that the attacker can find an approximately ground truth inference function of a two-layer neural network with ReLU activation function under Definition \ref{def:RelaPerInf}. The main body of the proof depends on \cite{MixCon}. Formally, consider the inversion problem
\begin{equation}
    f_\theta(x)=x_r,
\end{equation}
where $x_r\in \mathbb{R}^{o}$ is the received hidden layer representation, and without losing generality we normalize the input into $x \in [-1,1]^p$. $f_\theta$ is the device's net and we consider a two neural network with ReLU activation function, specifically,
\begin{equation}
    f_\theta(x)=W_2\sigma(W_1x+b),
\end{equation}
where $\sigma(\cdot)$ refers to ReLU function, and $W_1\in \mathbb{R}^{m_1\times p},\ W_2\in \mathbb{R}^{m_1\times o},\ b\in\mathbb{R}^{m_1}$ are linear transformation parameters. If an attacker can approximately recover the input data, then it can launch a successful ground truth inference attack. 

First of all, we state necessary definitions and useful related results.
\begin{defn}
\label{3SAT}
(\textsf{3SAT problem}). Given n variables and m clauses in a conjunctive normal form
CNF formula with the size of each clause at most 3, the goal is to decide whether there exists an
assignment to the n Boolean variables to make the CNF formula be satisfied.
\end{defn}
\begin{defn}
(\textsf{MAX3SAT}). Given n variables and m clauses, a conjunctive normal form CNF
formula with the size of each clause at most 3, the goal is to find an assignment that satisfies the
largest number of clauses.
\end{defn}
We use \textsf{MAXE3SAT(Q)} to denote the restricted special case of \textsf{MAX3SAT} where every variable occurs in at most Q clauses, and then we have,

\begin{thm}
(\cite{PreB1}). Unless \textsf{RP=NP}, there is no polynomial time (7/8 + 5$\sqrt{Q}$)-approximate algorithm for \textsf{MAXE3SAT(B)} .
\end{thm}

Now, we are ready to formally prove the Theorem \ref{thm:hardness}.

\begin{proof}
Given an \textsf{3SAT} instance $\omega$ with $n$ variables and $m$ clause, where each variable appears in at most Q clauses, we construct a two layer neural network $f_\omega$ and output representation $x_r$ satisfy the following:
\begin{itemize}
    \item Completeness. If $\omega$ is specifiable, then there exists $x\in[0,1]^{p}$ such that $f_\omega=x_r$
    \item Soundness. For any $x$ such that $\Vert f_\theta (x)-x_r\Vert_2\leq\frac{1}{60Q}\sqrt{o}$, we can recover an assignment to $\omega$ that satisfies at least $(7/8 + 5\sqrt{Q})m$ clauses
\end{itemize}
We set $p=n,\ m_1=m + 200Q^{2}n$ and $o=m+100Q^2n$. For any $j\in[m]$, we use $\omega_j$ to denote the $j$-th clause and use $f_{1,j}(x)$ to denote the output of the $j$-th neuron in the first layer, i.e., $f_{1,j}(x)=\sigma(W_j^{(1)}x+b_j)$, where $W_j^{(1)}$ is the $j$-th row of $W^{(1)}$. For any $i\in[n]$, we use $X_i$ to denote the $i$-th variable.

We use the input $x \in [−1, 1]^n$ to denote the variable, and the first $m$ neurons in the first layer to denote the $m$ clauses. By taking
\begin{equation}
    W_{j,i}^{(1)}=\left\{
    \begin{aligned}
    &1,    &&X_i\in \omega_j;\\
    &-1,      &&\overline{X_i} \in\omega_j;\\
    &0,       && otherwise,
    \end{aligned}
    \right.
\label{eq:GRR2}
\end{equation}
and $b_j=-2$ for any $i\in[n],\ j\in[m]$, and viewing $x_i=1$ as $X_i$ to be false and $x_i=-1$ as $X_i$ to be true. One can verify that $f_{1,j}(x)=0$ if the clause is satisfied, and $f_{1,j}(x)=1$ if the clause is unsatisfied. We simply copy the value in the second layer $f_j(x)=f_{1,j}(x)$ for $j\in[m]$. For other neurons, we make $100Q^2$ copies for each $|x_i (i \in n)|$ in the output layer. This can be achieved by taking
\begin{equation}
\begin{aligned}
    f_{m+(i−1)\cdot100Q^2+k}(x)=&f_{m+(i−1)\cdot100Q^2+k}(x)\\
    &+f_{1;m+100Q^2n+(i−1)\cdot100Q^2+k}(x),
\end{aligned}
\end{equation}
and set
\begin{equation}
\begin{aligned}
    f_{1,m+(i−1)\cdot100Q^2+k}(x)=\max\{ x_i,0\}\\
    f_{1,m+100Q^2n+(i−1)\cdot100Q^2+k}(x)=\max\{ -x_i,0\}.
\end{aligned}
\end{equation}
For any $i\in[n]$, $k\in[100Q^2]$. Finally, we set the target output as
\begin{equation}
    x_r =  (\underbrace{0,...,0}_m, \underbrace{1,...,1}_{100Q^2n})
\end{equation}
We are left to prove the two claims we made about the neural network $f_\theta$ and the target output
$x_r$. For the first claim, suppose $\omega$ is satisfiable and $X=(X_1,...,X_n)$ is the assignment. Then as argued before, we can simply take $x_i=1$ if $X_i$ is false and $x_i=-1$ is $X_i$ is true. One can check that $f_\theta(x)=x_r$.

For the second claim, suppose we are given $x\in[-1,1]^p$ such that
\begin{equation}
    \Vert f_\theta (x)-x_r\Vert_2\le\frac{1}{60Q}\sqrt{o}.
\end{equation}
We start from the simple case when $x$ is binary, i.e., $x \in \{-1, 1\}^{n}$. Again, by taking $X_i$ to be true if $x_i = −1$ and $X_i$ to be false when $x_i = 0$. One can check that the number of unsatisfied clause is at most
\begin{equation}
\begin{aligned}
    \Vert f_\theta (x)-x_r\Vert_2^2&\le\frac{1}{3600Q^2}o\\
    &=\frac{1}{3600Q^2}(m+100Q^2n)\\
    &\le \frac{1}{12}m+\frac{1}{3600Q^2}m\\
    &\le\frac{1}{8}m-\frac{5}{\sqrt{Q}}m.
\end{aligned}
\end{equation}
The third step follows from $n ≤ 3m$, and the last step follows from $B ≥ 15000$.

Next, we move to the general case that $x\in[-1,1]^{p}$. We would round $xi$ to −1 or +1 based on
the sign. Define $\overline{x}\in\{-1,1\}^n$ as
\begin{equation}
    \overline{x_i}=\arg \min_{t\in\{-1.1\}}|t-x_i|.
\end{equation}
We prove that $\overline{x}$ induces an assignment that satisfies $\left(\frac{7}{8}+\frac{5}{\sqrt{Q}}\right)m$ clauses. It suffices to prove
\begin{equation}
    \Vert f_\theta (\overline{x} )-x_r\Vert_2^2-\Vert f_\theta (\overline{x} )-x_r\Vert_2^2\le\frac{3}{100}m,
\end{equation}
since this implies the number of unsatisfied clause is bounded by
\begin{equation}
\begin{aligned}
     \Vert f_\theta (\overline{x} )-x_r\Vert_2^2&\leq(\frac{1}{12}m+\frac{1}{36Q^2}m)+\frac{3}{100}m\\
     &\leq \frac{1}{8}m-\frac{5}{\sqrt{Q}}m.
\end{aligned}
\end{equation}
The last step follows from $Q\ge10^7$.
Further, we define $\delta_i:=|\overline{x}-x_i|=1-|x_i|\in[0,1]$ and $T:=m+128Q^2n$. Then we have
\begin{equation}
    \begin{aligned}
        &\quad\Vert f_\theta (\overline{x} )-x_r\Vert_2^2-\Vert f_\theta (\overline{x} )-x_r\Vert_2^2\\
        &=\sum_{j=1}^{T}(f_j(\overline{x}-x_{rj}))-(f_j(x)-x_{rj})^2\\
        &=\sum_{j=1}^{m}(f_j(\overline{x})-x_{rj})^2-(f_j(x)-x_{rj})^2\\
        &\quad+\sum_{j=m+1}^{T}(f_j(x)-x_r)^2-(f_i(x)-x_r)^2\\
        &=\sum_{j=1}^{m}f_j(\overline{x})^2-f_j(x)^2-100Q^2\sum_{i=1}^n\delta^2_i\\
        &\le2\sum_{j=1}^{m}|f_{1,j}(\overline{x})-f_{1,j}(x)|-100Q^2\sum_{i=1}^{n}\delta^2_i\\
        &\le2\sum_{j=1}^{m}\sum_{i\not\in\omega_j}\delta_i-100Q^2\sum_{i=1}^{n}\delta^2_i\\
        &\le\frac{n}{100}\\
        &\le\frac{3m}{100}.
    \end{aligned}
\end{equation}
The third step follow from $x_rj = 0$ for $j \in [m]$ and for $j \in \{m + 1,\dots , m + 100Q^{2n}\}, x_rj = 1, \vert f_j(\overline{x}) − x_{rj}\vert_2 = 0$ and$ f_j(x) − x_{rj}\vert_2 = \delta_i $ given $j\in[m + (i − 1)\cdot 100Q^2 + 1, i\cdot100Q^2]$. The fourth step follows from that $f_j(x) = f_{1,j}(x)\in[0,1] for j\in[m]$. The fifth step follows from the 1-Lipschitz continuity of the ReLU. The sixth step follows from each variable appears in at most $Q$ clause. This concludes the second claim.
\end{proof}